\documentclass{article}
\usepackage[nowritingstyle]{preprint}
\usepackage{microtype}

\usepackage{amsmath,amsfonts,bm}

\def\eqref#1{equation~\ref{#1}}

\def\1{\bm{1}}

\DeclareMathAlphabet{\mathsfit}{\encodingdefault}{\sfdefault}{m}{sl}
\SetMathAlphabet{\mathsfit}{bold}{\encodingdefault}{\sfdefault}{bx}{n}

\usepackage{float}
\usepackage{subcaption}
\usepackage{latexsym}
\usepackage{enumitem}
\usepackage{pifont}
\usepackage{placeins}
\usepackage{amsfonts}
\usepackage{nicefrac}
\usepackage{adjustbox}
\usepackage{makecell}
\usepackage{multirow}
\usepackage{amsmath}
\usepackage{amssymb}
\usepackage{mathtools}
\usepackage{amsthm}
\usepackage{hyperref}            %
\usepackage{url}
\usepackage[capitalize,noabbrev]{cleveref} %

\theoremstyle{plain}
\newtheorem{theorem}{Theorem}[section]
\newtheorem{proposition}[theorem]{Proposition}
\newtheorem{lemma}[theorem]{Lemma}
\newtheorem{corollary}[theorem]{Corollary}
\theoremstyle{definition}
\newtheorem{definition}[theorem]{Definition}

\theoremstyle{remark}
\newtheorem{remark}[theorem]{Remark}

\newcommand{\best}[1]{\textbf{#1}}
\newcommand{\secon}[1]{\underline{#1}}

\newcommand{\oursfinal}{C\textsuperscript{2}U}

\newenvironment{legacyreported}{\begin{quote}\itshape\small [Paper-reported, original evaluation pipeline]~}{\end{quote}}

\title{C\textsuperscript{2}U in RAG: Compositional Concept Unlearning in Retrieval-Augmented Generation}

\author[1]{Yiyang Duan}
\author[1]{Fan Wu}
\author[1]{Tiantong Wu}
\author[1]{Che Wang}
\author[1]{Zhiren Gong}
\author[3]{Wei Guo}
\author[2]{Yang Cao}
\author[1]{Wei Yang Bryan Lim}
\affiliation[1]{Midori Lab, Nanyang Technological University}
\affiliation[2]{Institute of Science Tokyo (formerly Tokyo Institute of Technology)}
\affiliation[3]{Beihang University}
\runninghead{C2U in RAG}
\logos{%
  \logoimage[0.42in]{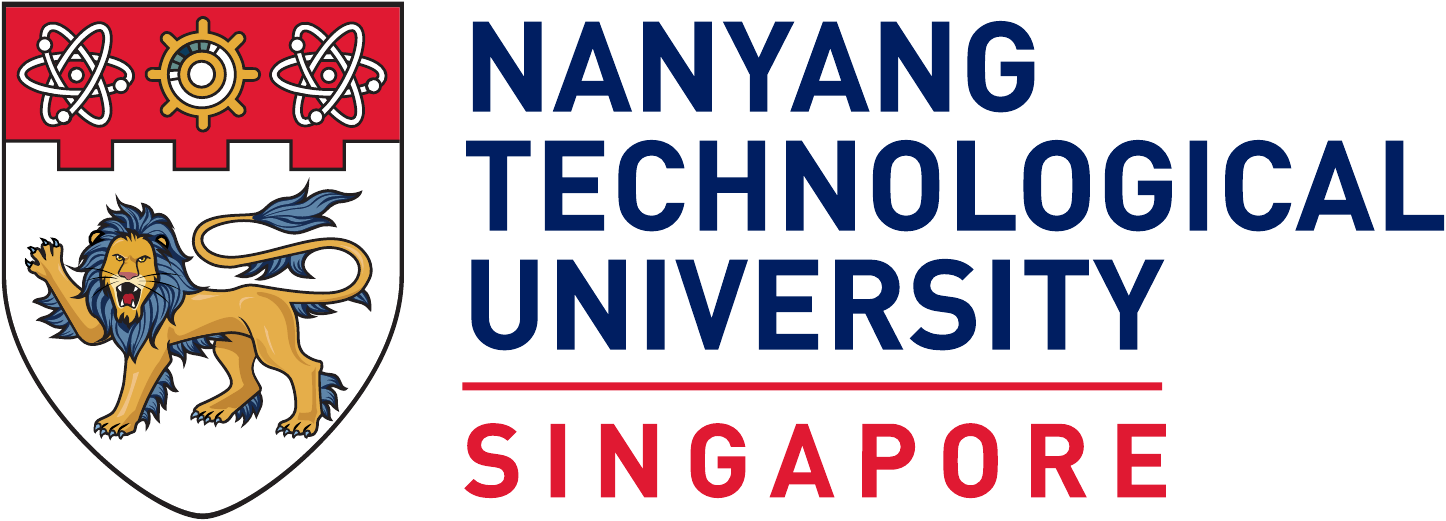}%
  \logodivider[0.42in]%
  \logoimage[0.42in]{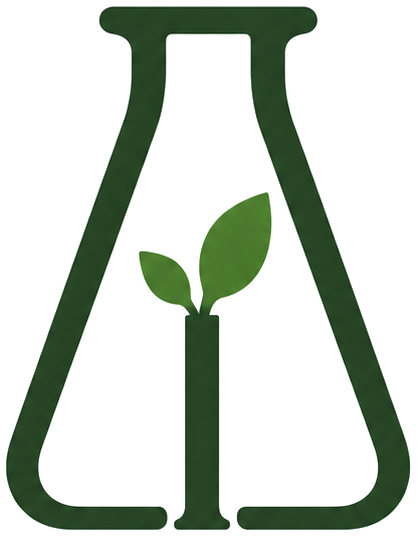}\hspace{0.6em}%
  \logoimage[0.13in]{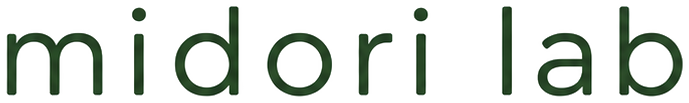}}

\begin{document}

\maketitle

\begin{abstract}
Retrieval-augmented generation can expose unwanted concepts through retrieved evidence even when the generator is moderated. We study query-side concept suppression with a fixed encoder and document index. Compositional Concept Unlearning (C\textsuperscript{2}U) routes a profile-derived direction through an augmented atom bank, adding only a bounded edited-minus-original delta to the raw query. Across five datasets, four encoders, and three seeds, the validation-selected configuration reduces pooled binary target Hit@5 by 4.25 percentage points with non-positive control degradation and target-exposure changes. Under the evaluated query transforms, it suppresses target density without increasing budgeted control-query exposure. Separately, ArXiv four-concept composition improves over an averaged-direction comparator after independent validation gates and first-use test evaluation.
\end{abstract}

\section{Introduction}
\label{sec:introduction}

Retrieval-Augmented Generation (RAG) grounds answers in external documents, but also creates a pathway for private, unwanted, or policy-restricted evidence to enter a model's context. Output moderation can reject a response without changing that evidence pathway. Conversely, deleting whole documents can remove useful content alongside an unwanted concept. We study an intermediate intervention: suppress concept-bearing retrieval through query-side representation editing, leaving the encoder and indexed corpus unchanged. Here, ``unlearning'' means retrieval-stage suppression, not certified parameter deletion or a guarantee about generated answers.

The problem has four aspects: intervention happens on a fixed-index retriever, a small profile must guide bounded query edits, suppression--utility trade-offs must hold under ordinary and transformed queries, and several concepts must share an editing budget. \cref{fig:paradigms} locates the intervention within RAG. Linear erasure provides useful comparisons, but reducing target retrieval need not reduce exposure under context packing. Access also matters: a query transform cannot inspect and remove target documents like a membership-aware filter. We therefore evaluate suppression, utility, and exposure together, distinguishing the information available to each method.

\begin{figure*}[t]
\centering
\includegraphics[width=\textwidth]{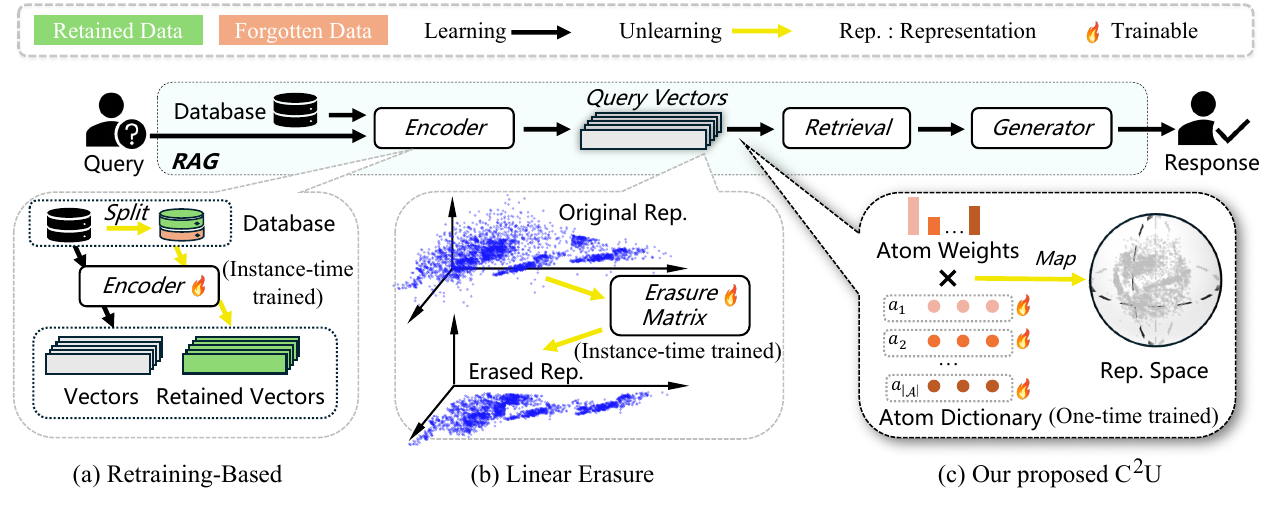}
\caption{\textbf{Intervention locations in RAG.} Retraining changes the representation model; linear erasure transforms representations; C\textsuperscript{2}U edits atom coordinates and bounds the resulting raw-query update. Its encoder, document index, and downstream generator remain fixed.}
\label{fig:paradigms}
\end{figure*}

C\textsuperscript{2}U combines an Atomic Compositional Representation (ACR) with a raw-query bypass. Target profiles and calibration controls determine a negative, diagonally preconditioned direction and a calibrated gate. Least-squares coefficients and an appended residual atom carry that direction through the frozen atom bank; only the weight difference returns to the raw query. A trust region limits the update. For multiple concepts, separate gates allocate a shared angular budget before their deltas are added.

The experiments assess both useful suppression and its limits. The single-concept comparison includes stronger-access methods and methods that suppress more at larger utility costs; query-transform evaluation tests whether suppression coexists with increased context exposure. Separately, ArXiv four-concept composition passes selection, independent confirmation, and first-use test evaluation, while the two-concept design fails selection. Our contributions are threefold:
\begin{itemize}
    \item \textbf{Bounded query editing.} An atom-coordinate operator with calibrated gating, residual completion, and a raw-query trust region (\cref{ssec:deployment_config}) with gate-off exact identity.
    \item \textbf{Suppression and utility evidence.} Joint retrieval and exposure comparisons across datasets, encoders, and query conditions, with explicit access classes and evidence scope.
    \item \textbf{Concept composition.} Separate-gate, shared-budget editing evaluated against averaged-direction steering.
\end{itemize}

\section{Related Work and Problem}
\label{sec:related_work}

\paragraph{Retrieval and safety.}
RAG research improves retrieval and context use through filtering, ranking, and corrective generation~\citep{hofstatter2023fid,asai2024self,yu2024rankrag,yan2024corrective,liu2024lost}. Retrieved sources can nevertheless introduce unsafe content or poisoning~\citep{an2025rag,zhao2025exploring}. Moderation and safety frameworks~\citep{inan2023llama,liang2025saferag} address complementary parts of this problem; adversarial suffixes and latent instructions illustrate why output defenses alone need evaluation~\citep{zou2023universal,qiu2023latent}. Unlike database redaction or retriever optimization~\citep{wang2025machine}, our intervention preserves the raw document index that retrieval searches.

\paragraph{Unlearning and representation control.}
Parameter-updating unlearning~\citep{chen2023boundary,tarun2023deep,fan2023salun} and diffusion concept erasure~\citep{orgad2023editing,gandikota2023erasing,wu2025erasing,bui2025fantastic} motivate selective interventions but do not establish fixed-index retrieval behavior. INLP~\citep{ravfogel2020null} and LEACE~\citep{belrose2023leace} remove linearly accessible concepts; CURE~\citep{kocak2025cure} and RAG Embedder Control~\citep{kim2025debiasing} provide recent representation-control comparisons. Our focus is a bounded query-side operator and its composition rule. Training and adaptation differences are described in \cref{app:baselines,sec:training_paradigm}.

\paragraph{Problem.}
\label{sec:background}
Let $q\in\mathbb{S}^{d-1}$ be a normalized frozen-encoder query and $u_j\in\mathbb{S}^{d-1}$ the unchanged indexed document vectors. Retrieval ranks $q^\top u_j$. Given a target concept profile, we produce $q'$ to reduce target retrieval while preserving control-label retrieval and avoiding extra target exposure on control queries. Query-role and target-document membership labels are evaluation inputs, not inputs to the deployed C\textsuperscript{2}U operator.

A profile identifies what to suppress through examples, not a runtime list of forbidden document IDs. Calibration controls characterize non-target geometry and set the gate before evaluation. At retrieval time the same operator applies to target and control queries without observing their roles. Suppressing target-seeking queries can still introduce unwanted evidence for other users, so target suppression and control-query exposure are separate outcomes throughout the evaluation.

\section{Deployed Query Editing}
\label{sec:methodology}
\label{ssec:c2u_overview}
\label{ssec:deployment_config}

\begin{figure*}[t]
\centering
\includegraphics[width=\textwidth]{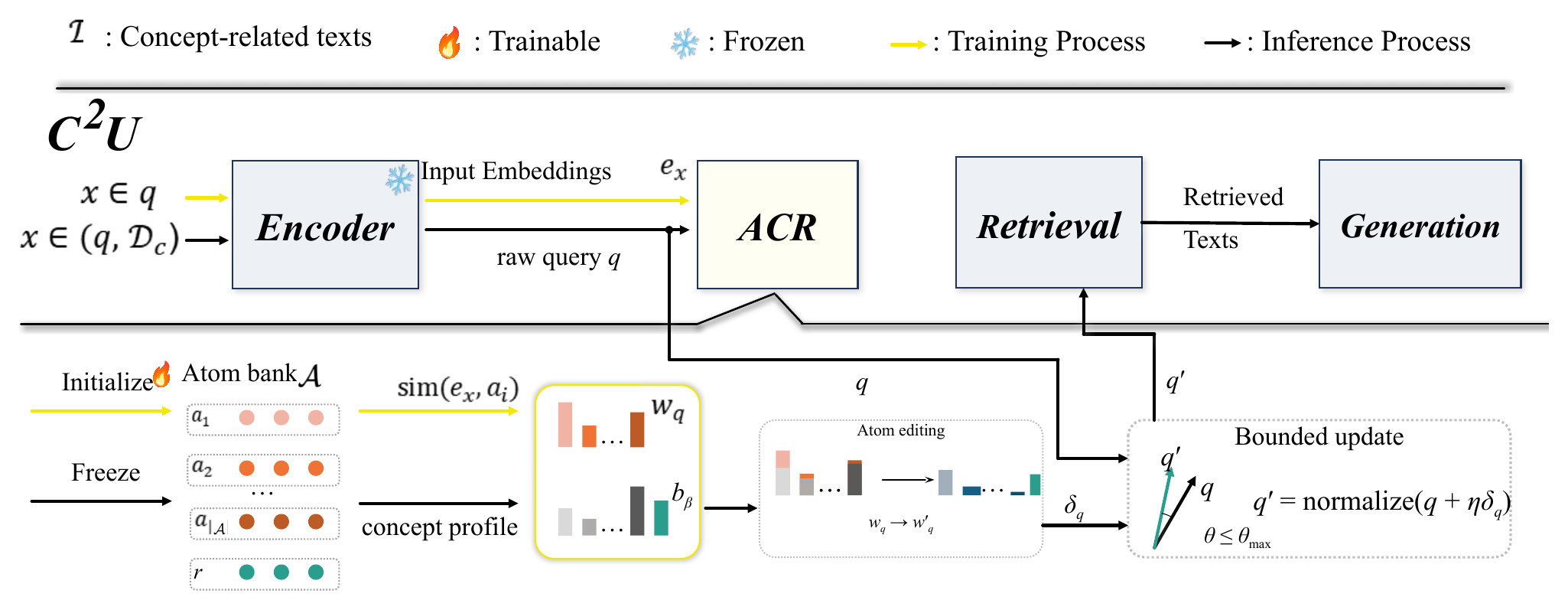}
\caption{\textbf{Deployed C\textsuperscript{2}U pipeline.} The frozen encoder's raw query $q$ feeds ACR and, via the bypass, the bounded update. Profile-guided editing adds the scaled coefficients $k_q b_\beta$ to the weights $w_q$; the difference $\delta_q$ and the update $q'=\operatorname{normalize}(q+\eta\delta_q)$, $\theta\le\theta_{\max}=0.2$~rad, return one edited query to the index.}
\label{fig:c2u_architecture}
\end{figure*}

\paragraph{Atomic representation.}
The atom bank $\mathcal A\in\mathbb R^{N\times d}$ is a learned dictionary whose rows $a_i$ are unlabeled unit directions in the encoder's representation space; $N$ varies by dataset and encoder (\cref{app:bank_dimensions}). The frozen encoder maps text $x$ to $h_x$, and a lightweight projection head scores $h_x$ against every row; a temperature-scaled softmax turns these scores into atom weights, whose bank combination reconstructs the representation:
\begin{equation}
\begin{gathered}
 w_x=\operatorname{softmax}\!\left(s(h_x,\mathcal A)/\tau\right),\\
 z_x=\operatorname{normalize}\!\left(\mathcal A^\top w_x\right).
\end{gathered}
\label{eq:acr_weights}
\end{equation}
Training optimizes only the bank and head---the encoder stays frozen throughout---with a retrieval objective on $z_x$ and an entropy penalty that keeps $w_x$ peaked; both are frozen afterward (\cref{ssec:ACR}). Editing then happens in atom coordinates: a direction in representation space becomes a change of weight coefficients rather than a free-form vector edit. The reconstruction itself never enters deployment. The index stores raw encoder vectors, and online editing returns only the weight-difference delta to the raw query rather than replacing the query by $z_x$; substituting the reconstruction would define a different method and break identity with the unchanged index deployment searches.

\paragraph{Direction and calibrated gate.}
Let $\mu_t,\mu_c$ be means of normalized target-profile and calibration-control vectors. With $v=\operatorname{normalize}(\mu_t-\mu_c)$, the negative, diagonally preconditioned direction is $g=\operatorname{normalize}(-D^{-1}v)$, where $D$ is a ridge-regularized, shrunk diagonal control second-moment matrix. A separate raw-space score $z(q)=q^\top\operatorname{normalize}(\mu_t)-q^\top\operatorname{normalize}(\mu_c)$ determines the continuous gate:
\begin{equation}
 s(q)=\operatorname{clip}\!\left(\frac{z(q)-z_0}{z_1-z_0},0,1\right).
\label{eq:raw_gate}
\end{equation}
Here $z_0,z_1$ are the calibration-control median and $95$th percentile. The gate uses neither ranks nor candidate labels; degenerate calibration uses the conservative fallback in \cref{app:method_details}.

\paragraph{Residual completion.}
The bank $\mathcal A\in\mathbb R^{N\times d}$ and weight-generating head are frozen after ACR training (\cref{ssec:ACR}); $N$ varies by dataset and encoder (\cref{app:bank_dimensions}). Solve $a=\operatorname{lstsq}(\mathcal A^\top,g)$ and set $p=\mathcal A^\top a$, $e=g-p$. Above the residual threshold, append $r=e/\|e\|$ and use $b_\beta=[a;\beta\|e\|]$ in $\mathcal A'=[\mathcal A;r]$. For zero or small residuals, append $r=p/\|p\|$ with coefficient zero; a degenerate $p$ fails closed. Thus $\beta$ scales \emph{only the appended coefficient}. At the selected $\beta=1$, projected and residual components recover full $g$ in exact arithmetic; the small-residual branch instead truncates $e$ to its in-span projection.

\paragraph{Angular allocation and delta transfer.}
Let $h=\mathcal A'^\top b_\beta$ and $\phi=\angle(q,h)$. The requested angle is $\theta_{\rm req}=0.3s(q)$; angular feasibility gives $\theta=\max(0,\min(\theta_{\rm req},\phi-\epsilon_{\rm plane}))$. For active, feasible rows set
\begin{equation}
 k_q=\frac{\sin\theta}{\sin(\phi-\theta)},\qquad
 w'_q=[w_q;0]+k_q b_\beta.
\label{eq:weight_edit}
\end{equation}
Edited weights are neither clipped nor renormalized. Normalizing $h$ defines its angle only; $b_\beta$ retains its magnitude. The requested, feasible, and realized angles therefore differ in general, particularly when $\beta\ne1$. Transfer the actual edited-minus-padded weight difference:
\begin{equation}
 \delta_q=(w'_q-[w_q;0])^\top\mathcal A'.
\label{eq:dense_delta}
\end{equation}
Training uses reconstruction, whereas online editing uses this difference; precision, cancellation, and threshold details are in \cref{app:method_details}.

\paragraph{Raw-query trust-region update.}
For active rows, return $q'=\operatorname{normalize}(q+\eta\delta_q)$, taking the full step if feasible and otherwise reducing $\eta\in[0,1]$ to satisfy the numerical counterpart of
\begin{equation}
 \langle q',q\rangle\ge\max(\cos(0.2),0.98).
\label{eq:trust_exit}
\end{equation}
The raw query, not its reconstruction, anchors the update (\cref{fig:c2u_architecture}). Gate-off and zero-delta rows return the input directly. The $0.2$-rad cap is tighter than the $0.98$ cosine floor and can further reduce the realized angle. Numerical checks are detailed in \cref{app:method_details}.

The trust region bounds motion, not retrieval damage. For ideal unit queries satisfying a cap $\theta_{\max}$, each unit document's score changes by at most $2\sin(\theta_{\max}/2)$ (\cref{prop:drift_bound}).

\paragraph{Design rationale.}
Each component answers one deployment question. The calibrated gate decides whether a query is edited at all, from the profile and unlabeled calibration controls alone: it needs neither candidate labels nor ranks, scales the requested angle continuously, and passes concept-distant queries through as an exact identity. The raw-query bypass fixes what is edited: the deployed object is the encoder's own vector, the index is never touched, and the reconstruction $z_x$ stays a training device. The residual atom makes the profile direction expressible in bank coordinates even when the frozen atom bank does not span it, so the weight-difference transfer carries the full direction rather than its projection, and the trust region bounds the resulting rotation. In composition, one shared budget caps the summed rotation across concepts, so several active gates cannot compound into a large motion, keeping the deployed operator bounded, query-side, and index-preserving.

\paragraph{Why not profile subtraction.}
The classical atom-coordinate edit subtracts profile mass from query weights (\cref{app:subtraction_formalization}). Under the locked bank and deployed runtime, every measured subtraction variant is an exact inactive identity: frozen query weights overlap the top-$30$ atom profile by only $\approx\!0.0011$, and the applied-angle diagnostic stays near $2\times10^{-5}$~rad. Two validation-only revival probes---knob-only bank retraining and an off-diagonal orthogonality loss---fail to revive subtraction and leave the deployed route unchanged (\cref{app:method_details}). Atom coordinates are what make this failure measurable: query and profile mass occupy almost disjoint weight coordinates, so a subtractive edit has nothing to act on. Suppression must instead add a bounded rotation along the completed direction---the deployed operator. These bounded interventions do not rule out every subtraction method under other banks or objectives.

\section{Evaluation}
\label{sec:experiments}
\subsection{Setup and Evidence Scope}
\label{sec:setup}

We evaluate Amazon, ArXiv, IMDB, MNLI, and TOFU~\citep{McAuley2013HiddenFA,bahraminasr2020imdb,williams2018broad,maini2024tofu}, using MiniLM, bge, e5, and Qwen3 frozen encoders and three seeds. Single-concept stages contain $60$ units, each with $100$ target and $300$ control queries, a shared $100$-target/$300$-control raw index, and $20$ target-profile documents. The final comparison contains $6{,}000$ target and $18{,}000$ control row contributions. Profile, selection, confirmation, and final-evaluation allocations are disjoint; selection and independent confirmation precede the fixed configuration.

\begin{table}[t]
\centering
\small
\setlength{\tabcolsep}{2.5pt}
\caption{\textbf{Task overview}: what each dataset suppresses and preserves, and its query source (\cref{app:dataset_details}). $^{g}$Relaxed global controls. $^{d}$One synthetic template per role.}
\label{tab:task_overview}
\begin{tabular}{@{}llll@{}}
\toprule
\textbf{Dataset} & \textbf{Target} & \textbf{Control} & \textbf{Queries} \\
\midrule
ArXiv & topic & category$^{g}$ & title/abstract \\
TOFU & sensitive topic & author identity & questions \\
Amazon & sentiment & product category & review titles \\
IMDB & movie genre & sentiment & template$^{d}$ \\
MNLI & text genre & logic label & hypotheses \\
\bottomrule
\end{tabular}
\end{table}

\cref{tab:task_overview} summarizes the task each dataset poses: a target concept whose documents should leave the top ranks is paired with a control attribute whose retrieval must survive the edit---a scientific topic against its ArXiv category, a sensitive biography topic against author identity, review sentiment against product category, movie genre against review sentiment, and source genre against the entailment label. Queries are titles or abstracts, natural questions, review titles, or hypotheses; IMDB uses one synthetic template per role. ArXiv controls come from the whole corpus rather than matched categories, so its control statistics describe a relaxed, global comparison.

Single-concept and attack evaluations are posthoc extensions of existing test units, without test tuning or an untouched-test claim. Multi-concept experiments use separate budgets and acceptance gates; only ArXiv $K=4$ reaches a separately gated first-use test. IMDB's row-bootstrap intervals are duplicate-template diagnostics, not strong semantic-sample evidence.

\paragraph{Access and metrics.}
C\textsuperscript{2}U uses profiles and unlabeled calibration controls without candidate access. Raw, Delta-Edit, PISCES, Random, and the portable DARE-GRL mapper are query-side comparisons; INLP, LEACE, CURE, and RAG-Emb-Ctrl additionally fit on $400$ labeled calibration examples per unit. Filter-and-Refill uses candidate/document membership; Rank-Router uses target-document membership as a label-assisted benchmark router. PISCES, DARE-GRL, CURE, and RAG Embedder Control are in-house diagnostic adaptations, not official reproductions~\citep{gurarieh2025precise,ganin2015unsupervised,yu2024language,kocak2025cure,kim2025debiasing}. Their access classes preclude a cross-class dominance interpretation.

Binary H5 Drop is the reduction in \emph{any-target-in-top-five} retrieval on target queries. CN is control own-label Hit@5 loss; TE is the change in any-target-in-top-five retrieval on control queries. Positive CN or TE is a cost; utility-clean means both rate changes are at most $10^{-9}$. We report rate differences in percentage points (pp), with target-cosine changes in the full results.

Dataset summaries average four encoders and three seeds; pooled summaries weight the five datasets equally. Absolute levels and per-dataset intervals are in \cref{tab:absolute_h5_ci}; statistical definitions and resampling scope are in \cref{app:experimental_details}.

\subsection{Single-Concept Results}
\label{sec:exp_concept_erasure}

\begin{table*}[t]
\centering
\small
\definecolor{arrProfile}{HTML}{2A9D8C}
\definecolor{arrLabeled}{HTML}{287271}
\definecolor{arrAdapt}{HTML}{264653}
\definecolor{arrCandidate}{HTML}{E9C46B}
\definecolor{arrMember}{HTML}{E66F51}
\definecolor{arrHero}{HTML}{E9C46B}
\caption{\textbf{Single-concept suppression and utility by dataset.} Percentage-point changes from Raw, averaged over four encoders and three seeds; Pooled weights datasets equally. Tints mark access classes; full results are in Appendix~\ref{app:full_table}.}
\label{tab:rag_comparison_concise}
\setlength{\tabcolsep}{4pt}
\begin{minipage}[t]{0.48\textwidth}
\begin{tabularx}{\linewidth}{Xrrr}
\toprule
Method & H5 Drop $\uparrow$ & CN $\downarrow$ & TE $\downarrow$ \\
\midrule
\multicolumn{4}{l}{\textbf{Amazon}} \\
\rowcolor{arrHero!25} C\textsuperscript{2}U & +6.33 & -0.64 & -12.44 \\
\rowcolor{arrProfile!10} Delta-Edit & 0.00 & 0.00 & -0.08 \\
\rowcolor{arrProfile!10} PISCES & -0.33 & -0.25 & +2.69 \\
\rowcolor{arrProfile!10} Random & +1.17 & -0.17 & +3.00 \\
\rowcolor{arrLabeled!10} INLP & +12.17 & -0.50 & +18.86 \\
\rowcolor{arrLabeled!10} LEACE & +36.58 & -0.61 & -0.83 \\
\rowcolor{arrLabeled!10} RAG-Emb-Ctrl & +68.83 & -0.69 & -33.53 \\
\rowcolor{arrLabeled!10} CURE & 0.00 & 0.00 & 0.00 \\
\rowcolor{arrAdapt!10} DARE-GRL & +22.50 & -0.11 & +38.53 \\
\rowcolor{arrCandidate!10} Filter-and-Refill$^{\ddagger}$ & +0.92 & -0.56 & -4.14 \\
\rowcolor{arrMember!10} Rank-Router$^{\S}$ & +7.00 & -0.64 & -13.83 \\
\midrule
\multicolumn{4}{l}{\textbf{ArXiv$^{r}$}} \\
\rowcolor{arrHero!25} C\textsuperscript{2}U & +5.67 & -1.25 & -7.64 \\
\rowcolor{arrProfile!10} Delta-Edit & 0.00 & +0.03 & 0.00 \\
\rowcolor{arrProfile!10} PISCES & -0.17 & -0.25 & +1.33 \\
\rowcolor{arrProfile!10} Random & +1.08 & +4.92 & +6.31 \\
\rowcolor{arrLabeled!10} INLP & +5.92 & +5.00 & +8.89 \\
\rowcolor{arrLabeled!10} LEACE & +1.00 & +2.19 & +16.19 \\
\rowcolor{arrLabeled!10} RAG-Emb-Ctrl & +69.50 & +5.56 & -21.03 \\
\rowcolor{arrLabeled!10} CURE & 0.00 & 0.00 & 0.00 \\
\rowcolor{arrAdapt!10} DARE-GRL & +22.67 & +68.58 & +58.61 \\
\rowcolor{arrCandidate!10} Filter-and-Refill$^{\ddagger}$ & +38.33 & -1.97 & -18.06 \\
\rowcolor{arrMember!10} Rank-Router$^{\S}$ & +5.67 & -1.08 & -7.64 \\
\midrule
\multicolumn{4}{l}{\textbf{IMDB$^{d}$}} \\
\rowcolor{arrHero!25} C\textsuperscript{2}U & +2.58 & 0.00 & 0.00 \\
\rowcolor{arrProfile!10} Delta-Edit & 0.00 & 0.00 & 0.00 \\
\rowcolor{arrProfile!10} PISCES & +1.83 & 0.00 & 0.00 \\
\rowcolor{arrProfile!10} Random & +2.42 & +2.03 & +23.11 \\
\rowcolor{arrLabeled!10} INLP & 0.00 & 0.00 & +16.67 \\
\rowcolor{arrLabeled!10} LEACE & +1.83 & 0.00 & +8.33 \\
\rowcolor{arrLabeled!10} RAG-Emb-Ctrl & +30.75 & 0.00 & +25.00 \\
\rowcolor{arrLabeled!10} CURE & 0.00 & 0.00 & 0.00 \\
\rowcolor{arrAdapt!10} DARE-GRL & +34.08 & +25.00 & +58.33 \\
\rowcolor{arrCandidate!10} Filter-and-Refill$^{\ddagger}$ & 0.00 & 0.00 & 0.00 \\
\rowcolor{arrMember!10} Rank-Router$^{\S}$ & +2.58 & 0.00 & 0.00 \\
\bottomrule
\end{tabularx}
\end{minipage}%
\hfill%
\begin{minipage}[t]{0.48\textwidth}
\begin{tabularx}{\linewidth}{Xrrr}
\toprule
Method & H5 Drop $\uparrow$ & CN $\downarrow$ & TE $\downarrow$ \\
\midrule
\multicolumn{4}{l}{\textbf{MNLI}} \\
\rowcolor{arrHero!25} C\textsuperscript{2}U & +3.42 & -1.08 & -6.03 \\
\rowcolor{arrProfile!10} Delta-Edit & 0.00 & +0.03 & +0.03 \\
\rowcolor{arrProfile!10} PISCES & -0.25 & -0.08 & +0.42 \\
\rowcolor{arrProfile!10} Random & +1.58 & +2.67 & +3.67 \\
\rowcolor{arrLabeled!10} INLP & +4.00 & +0.22 & +5.83 \\
\rowcolor{arrLabeled!10} LEACE & +1.25 & +0.06 & +7.94 \\
\rowcolor{arrLabeled!10} RAG-Emb-Ctrl & +42.58 & +1.33 & -34.25 \\
\rowcolor{arrLabeled!10} CURE & 0.00 & 0.00 & 0.00 \\
\rowcolor{arrAdapt!10} DARE-GRL & +35.42 & +27.72 & +14.50 \\
\rowcolor{arrCandidate!10} Filter-and-Refill$^{\ddagger}$ & +10.50 & -1.94 & -11.83 \\
\rowcolor{arrMember!10} Rank-Router$^{\S}$ & +3.83 & -1.22 & -6.36 \\
\midrule
\multicolumn{4}{l}{\textbf{TOFU}} \\
\rowcolor{arrHero!25} C\textsuperscript{2}U & +3.25 & -0.44 & -2.14 \\
\rowcolor{arrProfile!10} Delta-Edit & 0.00 & 0.00 & +0.03 \\
\rowcolor{arrProfile!10} PISCES & +0.42 & +0.08 & +0.53 \\
\rowcolor{arrProfile!10} Random & -0.92 & +0.94 & +0.92 \\
\rowcolor{arrLabeled!10} INLP & +6.42 & +6.47 & +9.00 \\
\rowcolor{arrLabeled!10} LEACE & +4.92 & -0.08 & +0.58 \\
\rowcolor{arrLabeled!10} RAG-Emb-Ctrl & +28.92 & -1.19 & -23.78 \\
\rowcolor{arrLabeled!10} CURE & 0.00 & 0.00 & 0.00 \\
\rowcolor{arrAdapt!10} DARE-GRL & +11.33 & +24.58 & +11.72 \\
\rowcolor{arrCandidate!10} Filter-and-Refill$^{\ddagger}$ & +4.00 & -1.19 & -7.31 \\
\rowcolor{arrMember!10} Rank-Router$^{\S}$ & +4.58 & -0.42 & -2.67 \\
\midrule
\multicolumn{4}{l}{\textbf{Pooled}} \\
\rowcolor{arrHero!25} C\textsuperscript{2}U & +4.25 & -0.68 & -5.65 \\
\rowcolor{arrProfile!10} Delta-Edit & 0.00 & +0.01 & -0.01 \\
\rowcolor{arrProfile!10} PISCES & +0.30 & -0.10 & +0.99 \\
\rowcolor{arrProfile!10} Random & +1.07 & +2.08 & +7.40 \\
\rowcolor{arrLabeled!10} INLP & +5.70 & +2.24 & +11.85 \\
\rowcolor{arrLabeled!10} LEACE & +9.12 & +0.31 & +6.44 \\
\rowcolor{arrLabeled!10} RAG-Emb-Ctrl & +48.12 & +1.00 & -17.52 \\
\rowcolor{arrLabeled!10} CURE & 0.00 & 0.00 & 0.00 \\
\rowcolor{arrAdapt!10} DARE-GRL & +25.20 & +29.16 & +36.34 \\
\rowcolor{arrCandidate!10} Filter-and-Refill$^{\ddagger}$ & +10.75 & -1.13 & -8.27 \\
\rowcolor{arrMember!10} Rank-Router$^{\S}$ & +4.73 & -0.67 & -6.10 \\
\bottomrule
\end{tabularx}
\end{minipage}%
\par\smallskip
\begin{minipage}{\textwidth}\small
Raw (zero differences) is omitted; positive CN or TE is a cost. $^{\ddagger}$Candidate/document access; $^{\S}$membership-assisted; $^{r}$relaxed global control; $^{d}$duplicate-template diagnostic.
\end{minipage}
\end{table*}

\paragraph{Suppression and utility.}
\cref{tab:rag_comparison_concise} shows larger C\textsuperscript{2}U suppression on Amazon and ArXiv than on MNLI and TOFU. Aggregate-derived pooled binary target Hit@5 decreases from $94.7333\%$ to $90.4833\%$, a $4.25$ pp drop rather than near-complete removal. H5 Drop is positive with non-positive CN and TE on every dataset; pooled CN is $-0.68$ pp and TE $-5.65$ pp.

\paragraph{Suppression--utility trade-offs.}
RAG-Emb-Ctrl reaches pooled H5 Drop $+48.12$ pp, with utility/exposure costs including IMDB TE $+25.00$ pp; CURE is a utility-gated no-op and does not suppress. DARE-GRL reaches $+25.20$ pp with CN $+29.16$ pp and TE $+36.34$ pp. Filter-and-Refill ($+10.75$ pp) and Rank-Router ($+4.73$ pp) exceed C\textsuperscript{2}U with utility-clean outcomes under stronger access. Full results and the Pareto view are in \cref{app:full_table,fig:pareto_panels}.

\paragraph{Encoder heterogeneity.}
The dataset means hide a wide encoder-level spread (\cref{tab:full_h5_drop,tab:full_cn_delta,tab:full_te_delta}). On Amazon, C\textsuperscript{2}U's H5 Drop ranges from $2.0$ pp with bge to $14.0$ pp with e5; ArXiv is the most uniform ($4.3$--$8.0$ pp), MNLI spans $1.7$--$6.7$ pp and TOFU $1.3$--$5.0$ pp; on IMDB only bge shows a drop ($10.3$ pp). e5 is the most responsive encoder on three datasets and Qwen3 the least on three. CN is non-positive in $19$ of $20$ dataset--encoder cells (MNLI with Qwen3: $+0.67$ pp) and TE in all $20$. The pooled $4.25$ pp thus summarizes behaviours that differ sevenfold across encoders.

\paragraph{Operating point.}
\label{ssec:operating_regime}
The deployed configuration is the validation-selected residual coefficient $\beta=1$ with the $0.2$-rad cap and $0.98$ cosine floor; the validation sweep over larger caps and smaller $\beta$ found no eligible alternative (\cref{app:method_details}).

\subsection{Analysis of Bounded Editing}
\label{ssec:editing_analysis}

\begin{figure*}[t]
\centering
\includegraphics[width=\textwidth]{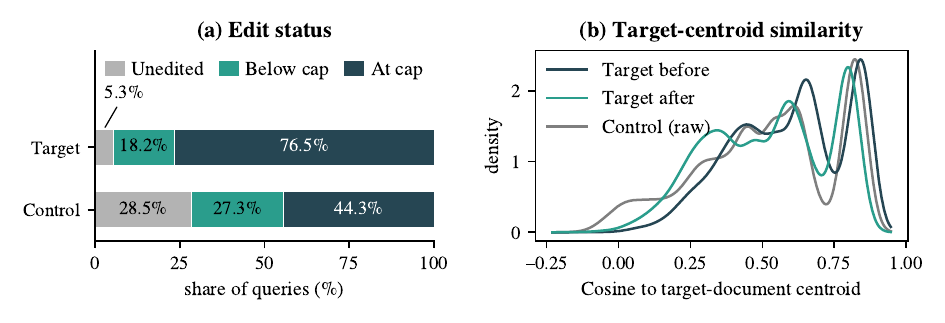}
\caption{\textbf{Bounded editing behavior.} (a) Share of target queries ($n{=}6{,}000$) and control queries ($n{=}18{,}000$) that are unedited, edited below the $0.2$~rad cap, or edited at the cap; each bar sums to its group. (b) Cosine to the target-document centroid for target queries before and after editing; \emph{Control (raw)} is the unedited control state.}
\label{fig:editing_analysis}
\end{figure*}

\cref{fig:editing_analysis} summarizes the bounded operator on the posthoc rows. The gate admits most queries: $94.7\%$ of target and $71.5\%$ of control queries receive an edit. The cap then binds: $76.5\%$ of target and $44.3\%$ of control queries sit at the $0.2$~rad boundary, so realized motion is usually below the requested angle. Panel (b) shows after-edit target queries at lower cosine to the target-document centroid. This view is consistent with the Hit@5 drops in \cref{tab:rag_comparison_concise}; the retrieval effect is measured there, not inferred from the shift. Two cautions remain: \emph{Control (raw)} is unedited while most control queries are edited, and centroid cosine is one direction, not full semantic alignment. Full angle distributions and retrieval examples: \cref{fig:applied_angle_hist,tab:case_study}.

\subsection{Suppression and Exposure under Query Transforms}
\label{subsec:exp2_attack}

\begin{table*}[t]
\centering
\scriptsize
\definecolor{arrProfile}{HTML}{2A9D8C}
\definecolor{arrLabeled}{HTML}{287271}
\definecolor{arrAdapt}{HTML}{264653}
\definecolor{arrCandidate}{HTML}{E9C46B}
\definecolor{arrMember}{HTML}{E66F51}
\definecolor{arrHero}{HTML}{E9C46B}
\caption{\textbf{Suppression and control exposure under query transforms.} Paired percentage-point changes from Raw, averaged over the three query conditions (Standard, Injection, Zou-noise) per dataset; Pooled weights the five datasets equally. Cell counts and signed worst values: Table~\ref{tab:attack_counts}; full per-condition matrix: Table~\ref{tab:attack_full}.}
\label{tab:attack_baselines_cer_compact}
\setlength{\tabcolsep}{4pt}
\begin{tabular}[t]{@{}l@{\hspace{1em}}l@{}}
\begin{tabular}[t]{@{}lrrr@{}}
\toprule
\multicolumn{4}{@{}l}{\textbf{Amazon}} \\
Method & Density $\Delta\downarrow$ & CER $\Delta\downarrow$ & Utility loss $\downarrow$ \\
\midrule
\rowcolor{arrHero!25} C\textsuperscript{2}U & -16.4 & -13.7 & 0.0 \\
\rowcolor{arrProfile!10} PISCES & -1.2 & +4.3 & 0.0 \\
\rowcolor{arrProfile!10} Random & -0.4 & +6.4 & 0.0 \\
\rowcolor{arrLabeled!10} INLP & -15.8 & +17.6 & 0.0 \\
\rowcolor{arrLabeled!10} LEACE & -21.9 & +2.4 & 0.0 \\
\rowcolor{arrLabeled!10} RAG-Emb-Ctrl & -34.7 & -31.4 & 0.0 \\
\rowcolor{arrLabeled!10} CURE & 0.0 & 0.0 & 0.0 \\
\rowcolor{arrAdapt!10} DARE-GRL & -9.2 & +46.0 & 0.0 \\
\bottomrule
\end{tabular}
&
\begin{tabular}[t]{@{}lrrr@{}}
\toprule
\multicolumn{4}{@{}l}{\textbf{MNLI}} \\
Method & Density $\Delta\downarrow$ & CER $\Delta\downarrow$ & Utility loss $\downarrow$ \\
\midrule
\rowcolor{arrHero!25} C\textsuperscript{2}U & -14.9 & -6.8 & +0.3 \\
\rowcolor{arrProfile!10} PISCES & -2.3 & +0.4 & -0.2 \\
\rowcolor{arrProfile!10} Random & -3.9 & +1.2 & +1.2 \\
\rowcolor{arrLabeled!10} INLP & -17.7 & +4.8 & -0.9 \\
\rowcolor{arrLabeled!10} LEACE & -8.4 & +3.3 & -0.3 \\
\rowcolor{arrLabeled!10} RAG-Emb-Ctrl & -40.6 & -47.5 & +5.7 \\
\rowcolor{arrLabeled!10} CURE & 0.0 & 0.0 & 0.0 \\
\rowcolor{arrAdapt!10} DARE-GRL & -29.8 & +6.6 & +6.9 \\
\bottomrule
\end{tabular}
\\[6pt]
\begin{tabular}[t]{@{}lrrr@{}}
\toprule
\multicolumn{4}{@{}l}{\textbf{ArXiv$^{r}$}} \\
Method & Density $\Delta\downarrow$ & CER $\Delta\downarrow$ & Utility loss $\downarrow$ \\
\midrule
\rowcolor{arrHero!25} C\textsuperscript{2}U & -17.7 & -11.7 & -2.2 \\
\rowcolor{arrProfile!10} PISCES & -2.0 & +0.6 & -0.2 \\
\rowcolor{arrProfile!10} Random & -5.6 & +4.6 & +4.7 \\
\rowcolor{arrLabeled!10} INLP & -34.1 & +2.0 & +4.1 \\
\rowcolor{arrLabeled!10} LEACE & -14.0 & +11.8 & +2.2 \\
\rowcolor{arrLabeled!10} RAG-Emb-Ctrl & -72.4 & -21.9 & +5.3 \\
\rowcolor{arrLabeled!10} CURE & 0.0 & 0.0 & 0.0 \\
\rowcolor{arrAdapt!10} DARE-GRL & -50.7 & +33.8 & +63.6 \\
\bottomrule
\end{tabular}
&
\begin{tabular}[t]{@{}lrrr@{}}
\toprule
\multicolumn{4}{@{}l}{\textbf{TOFU}} \\
Method & Density $\Delta\downarrow$ & CER $\Delta\downarrow$ & Utility loss $\downarrow$ \\
\midrule
\rowcolor{arrHero!25} C\textsuperscript{2}U & -4.4 & -4.5 & -0.1 \\
\rowcolor{arrProfile!10} PISCES & -0.6 & +1.3 & +0.1 \\
\rowcolor{arrProfile!10} Random & -0.1 & +1.5 & +0.9 \\
\rowcolor{arrLabeled!10} INLP & -5.6 & +14.0 & +2.2 \\
\rowcolor{arrLabeled!10} LEACE & -5.1 & +3.2 & 0.0 \\
\rowcolor{arrLabeled!10} RAG-Emb-Ctrl & -20.0 & -30.6 & +0.4 \\
\rowcolor{arrLabeled!10} CURE & 0.0 & 0.0 & 0.0 \\
\rowcolor{arrAdapt!10} DARE-GRL & -0.3 & +18.4 & +13.8 \\
\bottomrule
\end{tabular}
\\[6pt]
\begin{tabular}[t]{@{}lrrr@{}}
\toprule
\multicolumn{4}{@{}l}{\textbf{IMDB$^{d}$}} \\
Method & Density $\Delta\downarrow$ & CER $\Delta\downarrow$ & Utility loss $\downarrow$ \\
\midrule
\rowcolor{arrHero!25} C\textsuperscript{2}U & -11.1 & -0.8 & 0.0 \\
\rowcolor{arrProfile!10} PISCES & -0.6 & +0.2 & 0.0 \\
\rowcolor{arrProfile!10} Random & -5.4 & +14.7 & +0.7 \\
\rowcolor{arrLabeled!10} INLP & -16.7 & +10.1 & -0.1 \\
\rowcolor{arrLabeled!10} LEACE & -5.6 & +5.7 & +0.1 \\
\rowcolor{arrLabeled!10} RAG-Emb-Ctrl & -34.3 & +8.9 & -0.1 \\
\rowcolor{arrLabeled!10} CURE & 0.0 & 0.0 & 0.0 \\
\rowcolor{arrAdapt!10} DARE-GRL & -29.4 & +27.9 & 0.0 \\
\bottomrule
\end{tabular}
&
\begin{tabular}[t]{@{}lrrr@{}}
\toprule
\multicolumn{4}{@{}l}{\textbf{Pooled$^{p}$}} \\
Method & Density $\Delta\downarrow$ & CER $\Delta\downarrow$ & Utility loss $\downarrow$ \\
\midrule
\rowcolor{arrHero!25} C\textsuperscript{2}U & -12.9 & -7.5 & -0.4 \\
\rowcolor{arrProfile!10} PISCES & -1.4 & +1.4 & -0.1 \\
\rowcolor{arrProfile!10} Random & -3.1 & +5.7 & +1.5 \\
\rowcolor{arrLabeled!10} INLP & -18.0 & +9.7 & +1.1 \\
\rowcolor{arrLabeled!10} LEACE & -11.0 & +5.3 & +0.4 \\
\rowcolor{arrLabeled!10} RAG-Emb-Ctrl & -40.4 & -24.5 & +2.2 \\
\rowcolor{arrLabeled!10} CURE & 0.0 & 0.0 & 0.0 \\
\rowcolor{arrAdapt!10} DARE-GRL & -23.9 & +26.5 & +16.9 \\
\bottomrule
\end{tabular}
\\
\end{tabular}
\par\smallskip
\begin{minipage}{\textwidth}\small
Density: target-label fraction in the target-query top-5; CER: any-target occurrence in the packed control context (budget $B{=}512$, at most $50$ documents); utility: control own-label Hit@10 loss. CURE never suppresses. $^{r}$ArXiv: relaxed global control. $^{d}$IMDB: duplicate-template diagnostic. $^{p}$Pooled: equal-weight mean over the five datasets (fifteen dataset--condition cells). Metric definitions and posthoc scope: \S\ref{subsec:exp2_attack}.
\end{minipage}
\end{table*}

We compare each method with Raw under the same condition: standard queries, instruction injection, or Zou-noise suffixes~\citep{zou2023universal}. \emph{Target density $\Delta$} is the change in the fraction of top-five documents carrying the target label, averaged over target queries; it is not binary H5 Drop. \emph{CER $\Delta$} records whether any target document enters the packed context ($B=512$ tokens, at most $50$ documents), averaged over \emph{control queries}. Attack utility is control own-label Hit@10 loss, distinct from CN and TE. The two describe different populations, so a decrease in one need not entail a decrease in the other; both are reported together.

\cref{tab:attack_baselines_cer_compact} shows why the two must be read together. INLP lowers Amazon target density by $15.8$ pp while raising control exposure by $17.6$ pp: target documents leave the target queries' top five but enter other users' packed contexts. RAG-Emb-Ctrl lowers both on four datasets, most strongly on ArXiv ($-72.4$ and $-21.9$ pp), but raises IMDB exposure ($+8.9$ pp) and costs control utility ($+5.3$ pp on ArXiv, $+5.7$ on MNLI). C\textsuperscript{2}U's density and CER means are negative on every dataset with utility changes between $-2.2$ and $+0.3$ pp.

\cref{tab:attack_baselines_cer_compact} reports per-dataset means over the three conditions and an equal-weight pooled block. \cref{tab:attack_counts} counts the $15$ dataset--condition cells on unrounded values (per-condition matrix: \cref{tab:attack_full}), since an average can hide a condition-level increase: suppression requires density $\Delta<0$; divergence additionally requires CER $\Delta>0$. Counts are descriptive signs, not significance tests; standard is a no-attack condition. C\textsuperscript{2}U suppresses in all $15$ cells with $0/15$ divergent. Its attack utility change still ranges from $-5.2$ to $+2.6$ pp, so this does not establish universal utility preservation under attack.

INLP and LEACE diverge in $13/15$ cells; PISCES, DARE-GRL, and Random in $11/15$. RAG-Emb-Ctrl suppresses in all cells but diverges in the three IMDB cells (CER up to $+16.7$ pp); the duplicate-template limitation applies. CURE's $0/15$ divergence is trivial: it suppresses in zero cells. Among evaluated suppressing profile-only methods, only C\textsuperscript{2}U avoids density--exposure divergence under these query transforms. Stronger-access filters and routers are excluded by the attack contract. The full matrices and both posthoc arms also appear in \cref{app:attack_protocol}.

\subsection{Multi-Concept Composition}
\label{ssec:multi_concept_main}

Each concept $c$ has its own direction $g_c$, gate $s_c(q)$, and augmented-bank coefficients. Angular requests start at $0.3s_c(q)$ and shrink proportionally whenever their sum exceeds $\min(0.3\max_c s_c(q),0.2)$. Each allocated angle converts to a feasible coefficient scale; the atom-coordinate deltas are added and pass through the shared trust-region update. The comparator instead steers along the normalized average of unit concept directions under a single max gate---a different gating and allocation rule.

\begin{figure}[t]
\centering
\includegraphics[width=\columnwidth]{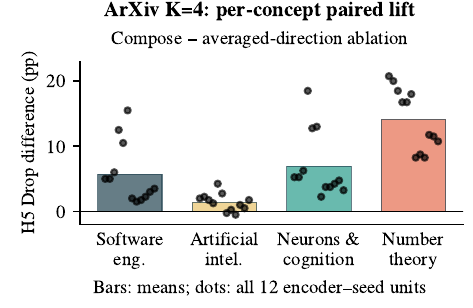}
\caption{\textbf{ArXiv $K=4$ per-concept H5 lift.} Compose minus the averaged-direction comparator: means and all twelve encoder--seed unit differences; these conceptwise differences are not the min-concept endpoint reported in the text.}
\label{fig:multi_concept_lift}
\end{figure}

The capacity-expanded ArXiv $K=4$ arm passes selection and independent confirmation before first access to fresh test rows. Its endpoint is
\begin{equation}
 \frac1{12}\sum_u\left[\min_c H^{\rm compose}_{u,c}
                    -\min_c H^{\rm avg}_{u,c}\right],
\label{eq:multi_endpoint}
\end{equation}
where $H$ is per-concept H5 Drop. Test lift is $+0.04125$, CI95 $[+0.0325,+0.050625]$, from $4{,}000$ bootstraps of the twelve paired unit differences. All four compose concept cells are positive; \cref{fig:multi_concept_lift} separately shows paired conceptwise differences. Filter-and-Refill-multi remains a non-gating document-membership oracle reference ($+0.7181$).

Independent gates decide, per concept, whether a query carries evidence for that concept and how much rotation it requests, so a query tied to one concept is edited mainly toward that concept. The shared budget rescales all requests so their sum stays within the single-concept budget: when several gates fire, each receives a smaller share. The per-concept lifts in \cref{fig:multi_concept_lift} are all positive but uneven: $+1.4$ pp for Artificial Intelligence (positive in $10$ of $12$ units), $+5.7$ for Software Engineering, $+6.9$ for Neurons and Cognition, and $+14.2$ for Number Theory. The spread is consistent with concepts differing in gate activity and direction separation, not with properties of individual atoms.

The expansion uses larger budgets than earlier arms, so the ladder is not a controlled experiment on sample size alone. It is the only untouched-test arm here: its $K=2$ design fails selection, and the earlier arms are validation-only or fail acceptance, as \cref{app:multi_concept,fig:concept_composition_viz} document.

\section{Conclusion}
\label{sec:conclusion}

C\textsuperscript{2}U edits the query, not the encoder or the index: a calibrated gate admits concept-evidenced queries, the profile direction enters atom coordinates via a residual atom, and the difference rotates the raw query by at most $0.2$~rad before retrieval. The selected configuration lowers target Hit@5 by $4.25$ pp pooled with non-positive control costs on every dataset mean, keeps density and control exposure changes non-positive in every dataset--condition cell under query transforms, and composes four ArXiv concepts with positive lift on each. It suits fixed-index retrievers tolerating partial, bounded suppression rather than removal.

\FloatBarrier
\section*{Limitations}
\label{sec:limitations}

Three limitations bound the claims. First, suppression is partial: the aggregate-derived after-edit target Hit@5 stays near $90\%$ pooled, so C\textsuperscript{2}U reduces exposure rather than removing a concept. Second, the data and evaluation scope is narrow: five English datasets, four encoders, and $100$/$300$-query units; IMDB has one template per role, ArXiv uses relaxed global controls, single-concept and attack results are posthoc extensions of frozen test units, and the untouched multi-concept claim covers ArXiv $K=4$ only, with capacity and granularity limiting other datasets. Third, we evaluate retrieval only: generated-output safety, multilingual or open-domain corpora, and defense-aware adaptive attackers are untested, and low runtime access does not establish a measured cost advantage.

\section*{Ethical considerations}
\label{sec:broader_impact}

Retrieval-side suppression may reduce unwanted exposure but can also hide legitimate evidence or enforce biased policies. Concept definitions, access to profiles, and effects on non-target retrieval require transparent governance and auditing. This method is not a substitute for privacy deletion, consent, output safeguards, or monitoring. The work evaluates retrieval rather than safety of generated content.

\bibliography{bib/references}
\bibliographystyle{preprint}

\newpage
\appendix

\section{Experimental Details}
\label{app:experimental_details}

\subsection{Implementation and Data Protocol}

\paragraph{Retrieval framework.}
All locked experiments share one evaluation pipeline per unit: the raw frozen-encoder document embeddings are encoded once per unit, shared bit-identically across methods, and never mutated; retrieval similarity is the dot product of $\ell_2$-normalized embeddings. The four encoders are \texttt{sentence-transformers/all-MiniLM-L6-v2}, \texttt{BAAI/bge-base-en-v1.5}, \texttt{intfloat/e5-base-v2}, and \texttt{Qwen/Qwen3-Embedding-0.6B}, each pinned to a specific revision with a local snapshot manifest.

\paragraph{Index construction and budgets.}
The single-concept and attack cohorts use $100$ target/$300$ control index documents and $100$ target/$300$ control queries per unit, with $20$ target-profile documents. These budgets do not describe all multi-concept stages. Within the single-concept allocation, profile, selection, confirmation, and test rows are disjoint. Multi-concept budgets and the known v2 overlap defect are reported in \cref{app:multi_concept}; blocked capacity rows are not zero-filled or replenished from another split.

\paragraph{Lock ladder and reuse.}
Single-concept selection and independent confirmation precede its method lock. Its final evaluation and the attacks extend previously frozen test units without test tuning or an untouched-test claim. Recent baselines have their own selection, confirmation, and posthoc-final records; their provenance-complete replay occurred after final outcomes were observed and closes provenance, not prospective selection. Multi-concept test access requires its separately specified acceptance lock. Only expanded ArXiv $K=4$ first accesses its allocated test rows after passing that gate. Common infrastructure does not make these one protocol or imply identical bootstrap populations across arms.

\paragraph{Hyperparameters and bank dimensions.}
\label{app:bank_dimensions}
The deployed single-concept setting appends one row to each frozen $N$-atom bank, with $20$ target-profile documents, residual-coefficient multiplier $\beta=1$, cap $0.2$ rad, and floor $0.98$. The frozen validation assets have $N=64$ for ArXiv under every encoder. For Amazon, IMDB, MNLI, and TOFU, $N=768$ for MiniLM, $1{,}536$ for bge/e5, and $2{,}048$ for Qwen3; embedding dimensions are respectively $384$, $768$, and $1{,}024$. These are asset dimensions, not a newly selected bank-size sweep. The historical top-$30$ atom-profile truncation is not the deployed raw-space gate. Configuration selection uses validation only; the posthoc evaluation does not retune it.

\subsection{Evaluation Metrics}

Let $\mathcal{Q}_t$ and $\mathcal{Q}_c$ denote the sets of target queries and control queries, respectively. For a given query $q$, $\text{TopK}(q)$ represents the set of labels or categories associated with the top-$K$ retrieved documents. H5, TC, and CN are reductions (Base minus Method), while TE is an exposure change (Method minus Base); the equations below specify each sign. The locked metric names map as follows: Hit@$K$ Drop $=$ \texttt{hit\_at\_5\_drop} (H5), Target Cos Drop $=$ TC, Control/Neighbor Delta $=$ CN, Target Exposure $\Delta$ $=$ TE.

\paragraph{Erasure and utility in the standard setting.}
\begin{enumerate}
    \item \textbf{Target Hit@K Drop ($\uparrow$)} measures the reduction in retrieval success for target queries:
    \begin{equation}
    \begin{gathered}
        \text{Hit@K Drop} = \frac{1}{|\mathcal{Q}_t|} \sum_{q \in \mathcal{Q}_t} \\
        \Bigl( \mathbf{1}[t \in \text{TopK}_{\text{Base}}(q)] \\
        {} - \mathbf{1}[t \in \text{TopK}_{\text{Method}}(q)] \Bigr).
    \end{gathered}
    \end{equation}

    \item \textbf{Target Cos Drop ($\uparrow$)} quantifies the reduction in cosine similarity between target-query representations and the target concept centroid. Let $\mu_t$ be the centroid of all target document embeddings in the base index:
    \begin{equation}
    \begin{gathered}
        \Delta_{\cos} = \frac{1}{|\mathcal{Q}_t|} \sum_{q \in \mathcal{Q}_t} \\
        \Bigl( \cos(\mathbf{v}_{\text{base}}(q), \mu_t) \\
        {} - \cos(\mathbf{v}_{\text{method}}(q), \mu_t) \Bigr).
    \end{gathered}
    \end{equation}

    \item \textbf{Control/Neighbor Delta ($\downarrow$)} evaluates the impact on non-target retrieval utility. Let $y_q$ be the ground-truth control label for $q \in \mathcal{Q}_c$:
    \begin{equation}
    \begin{gathered}
        \text{Utility Delta} = \frac{1}{|\mathcal{Q}_c|} \sum_{q \in \mathcal{Q}_c} \\
        \Bigl( \mathbf{1}[y_q \in \text{TopK}_{\text{Base}}(q)] \\
        {} - \mathbf{1}[y_q \in \text{TopK}_{\text{Method}}(q)] \Bigr).
    \end{gathered}
    \end{equation}

    \item \textbf{Target Exposure $\Delta$ ($\downarrow$)} measures the change in false-positive retrieval of target documents by control queries:
    \begin{equation}
    \begin{gathered}
        \text{Exposure } \Delta = \frac{1}{|\mathcal{Q}_c|} \sum_{q \in \mathcal{Q}_c} \\
        \Bigl( \mathbf{1}[t \in \text{TopK}_{\text{Method}}(q)] \\
        {} - \mathbf{1}[t \in \text{TopK}_{\text{Base}}(q)] \Bigr).
    \end{gathered}
    \end{equation}
\end{enumerate}

All three locked main hit metrics use $K=5$: H5 and TE concern target-label presence, while CN concerns control own-label presence. Attack utility instead uses own-label Hit@$10$. Rate differences in percentage points multiply the above quantities by $100$; target-cosine differences are not rates.

\paragraph{Absolute levels and intervals.}
The main table's Raw and C\textsuperscript{2}U after-rates are aggregate-derived. Per dataset, Raw H5 is reconstructed as a comparator's after-hit rate plus its locked H5 Drop; all seven reconstructions agree to floating-point precision. C\textsuperscript{2}U's own drop is then subtracted. The five Raw/after pairs are $96.8333/90.5000\%$ (Amazon), $97.8333/92.1667\%$ (ArXiv), $100.0000/97.4167\%$ (IMDB), $94.5833/91.1667\%$ (MNLI), and $84.4167/81.1667\%$ (TOFU). Equal-weight pooling gives $94.7333/90.4833\%$. These are not new measurements. The per-dataset binary H5 Drop intervals reuse the existing standard/no-attack matched replay and its $4{,}000$ paired-row bootstrap samples. There is no native pooled H5 CI in the final summary; we neither average interval endpoints nor substitute validation intervals. IMDB row resampling does not create independent templates.

\paragraph{Robustness under query transforms.}
We evaluate Standard (no attack), Injection, and Zou-noise. Changes are relative to Raw under the same condition, paired by query.
\begin{itemize}
    \item \textbf{Target density $\Delta$ (pp) ($\downarrow$)}: for $D_5(q)=\frac15\sum_{j\in\operatorname{Top5}(q)}\mathbf{1}[j\in\mathcal D_t]$, this is $100\mathbb E_{\mathcal Q_t}[D_5^{\rm Method}-D_5^{\rm Raw}]$. It is a target-document fraction, not binary H5.
    \item \textbf{CER $\Delta$ (pp) ($\downarrow$)}: let $\operatorname{Docs}(q;512,50)$ be the ranked context packed under the locked token accounting, with budget $512$, at most $50$ documents, and partial last-document inclusion where allowed. Then
    \begin{equation}
    \operatorname{CER}(q)=\mathbf{1}[\operatorname{Docs}(q;512,50)\cap\mathcal D_t\ne\emptyset].
    \end{equation}
    The reported delta is $100\mathbb E_{\mathcal Q_c}[\operatorname{CER}^{\rm Method}-\operatorname{CER}^{\rm Raw}]$: the population is \emph{control queries}.
    \item \textbf{Utility Drop (pp) ($\downarrow$)}: the decrease in control own-label Hit@$10$ under the same condition. CN instead uses own-label Hit@$5$; TE uses target Hit@$5$.
\end{itemize}
Suppression and divergence counts use unrounded dataset--condition signs, not significance tests.

\begin{table*}[t]
\centering
\small
\caption{\textbf{Absolute target Hit@5 and dataset uncertainty.} Raw and C\textsuperscript{2}U levels are aggregate-derived; intervals describe C\textsuperscript{2}U's binary H5 Drop.}
\label{tab:absolute_h5_ci}
\setlength{\tabcolsep}{6pt}
\begin{tabular}{@{}lrrrr@{}}
\toprule
Dataset & Raw H5 (\%) & C\textsuperscript{2}U H5 (\%) & Drop (pp) & 95\% CI (pp) \\
\midrule
Amazon & 96.83 & 90.50 & +6.33 & [4.83, 7.92] \\
ArXiv$^{r}$ & 97.83 & 92.17 & +5.67 & [4.42, 7.00] \\
IMDB$^{d}$ & 100.00 & 97.42 & +2.58 & [1.75, 3.50] \\
MNLI & 94.58 & 91.17 & +3.42 & [2.25, 4.50] \\
TOFU & 84.42 & 81.17 & +3.25 & [1.92, 4.58] \\
\bottomrule
\end{tabular}
\par\smallskip
\begin{minipage}{\textwidth}\small
Intervals reuse the existing standard/no-attack paired-row replay (4,000 bootstrap samples; 1,200 target-row contributions per dataset). No pooled H5 interval is reported. $^{r}$Relaxed global control. $^{d}$IMDB has one unique template per target/control role; its row-bootstrap interval is descriptive, not uncertainty over independent semantic examples.
\end{minipage}
\end{table*}

\subsection{ACR Training and Reconstruction}
\label{sec:training_paradigm}
\label{ssec:ACR}

ACR trains a dictionary and lightweight projection head atop a frozen encoder. Atoms have no pre-assigned semantic labels and start from standard-normal samples. Given frozen representation $h_x$, a bilinear compatibility score and temperature-scaled softmax give
\begin{equation}
\label{eq:semantic_recipe_flat}
 w_{x,i}=\frac{\exp(\operatorname{sim}(h_x,a_i)/\tau)}{\sum_j\exp(\operatorname{sim}(h_x,a_j)/\tau)}.
\end{equation}
The weights are peaked but not literally sparse. With rows $a_i$ in $\mathcal A$, training reconstruction is
\begin{equation}
 z_x=\frac{\mathcal A^\top w_x}{\|\mathcal A^\top w_x\|_2}
     =\frac{\sum_iw_{x,i}a_i}{\|\sum_iw_{x,i}a_i\|_2}.
\end{equation}
This reconstruction is used by the training objective. At deployed retrieval the index contains raw encoder vectors, and the raw query receives only \eqref{eq:dense_delta}; replacing it by $z_x$ would define a different method. Entropy concentration does not prove that the deployed additive edit has small support. The full training objective follows.

\subsection{Optimization Setup}
\label{sec:optimization_setup}

\paragraph{Training dynamics.}
To preserve the retrieval geometry of the pre-trained dense retriever, we freeze the encoder parameters throughout ACR training. Optimization is restricted to the atom bank $\mathcal{A}$ and the lightweight projection head. We use AdamW with a batch size of $32$ and an initial learning rate of $10^{-3}$, a \texttt{ReduceLROnPlateau} scheduler (factor $0.5$, patience one epoch), a $90/10$ train-validation split, and early stopping within a small epoch budget.

\paragraph{Hyperparameter configuration.}
The reported training setup uses alignment penalty weight $\lambda_{\mathrm{align}}=1.0$, entropy regularization weight $\lambda_{\mathrm{ent}}=0.02$, retrieval margin $m_r=0.3$, and scoring temperature $\tau=0.1$. Bank dimensions vary as specified in \cref{app:bank_dimensions}; the locked reproduction probe establishes bit-identical recovery of its MNLI/MiniLM/seed-1 checkpoint, not one universal checkpoint or bank size for all units.

\paragraph{Training objective (ACR).}
We freeze the transformer encoder and optimize only the atom bank and a lightweight projection head. For each triplet $(q,d^+,d^-)$, ACR outputs a composed Euclidean embedding $e^{\mathrm{acr}}_x$ and atom weights $w_x\in\Delta^{|\mathcal{A}|-1}$ for $x\in\{q,d^+,d^-\}$. We $L_2$-normalize the composed embeddings and compute cosine similarities by dot product:
$$
z_x = \frac{e^{\mathrm{acr}}_x}{\|e^{\mathrm{acr}}_x\|_2+\varepsilon},
\qquad
\cos(z_a,z_b)= z_a^\top z_b.
$$
Let $e^{\mathrm{base}}_x$ denote the frozen dense retriever embedding for input $x$, and let $\bar e^{\mathrm{base}}_x = e^{\mathrm{base}}_x/(\|e^{\mathrm{base}}_x\|_2+\varepsilon)$. The total loss is:
$$
\mathcal{L}
=
\mathcal{L}_{\mathrm{tri}}
+
\lambda_{\mathrm{align}}\mathcal{L}_{\mathrm{align}}
+
\lambda_{\mathrm{ent}}\mathcal{L}_{\mathrm{ent}},
$$
where:

\noindent\textbf{(i) Cosine triplet retrieval loss.}
\begin{equation*}
\begin{split}
\mathcal{L}_{\mathrm{tri}}
=\max\bigl(0,\;&m_r - \cos(z_q,z_{d^+})\\
&{}+\cos(z_q,z_{d^-})\bigr).
\end{split}
\end{equation*}

\noindent\textbf{(ii) Cosine alignment loss.}
We encourage each ACR-composed embedding to remain close to the corresponding frozen retriever embedding:
$$
\mathcal{L}_{\mathrm{align}}
=
\sum_{x\in\{q,d^+,d^-\}}
\Bigl(1-\cos(z_x,\bar e^{\mathrm{base}}_x)\Bigr).
$$

\noindent\textbf{(iii) Entropy penalty on atom weights.}
To obtain sharper training-weight distributions, we penalize entropy; this does not imply sparsity of the deployed least-squares coefficient edit:
$$
\mathcal{L}_{\mathrm{ent}}
=
\sum_{x\in\{q,d^+,d^-\}}
\Bigl(-\sum_{k=1}^{|\mathcal{A}|} w_{x,k}\log(w_{x,k}+\epsilon)\Bigr).
$$
Here, $m_r>0$ is the triplet margin, and $\lambda_{\mathrm{align}},\lambda_{\mathrm{ent}}>0$ are scalar weights.

\subsection{Historical Profile Subtraction (Not the Deployed Edit)}
\label{sec:inference_erasure}
\label{app:subtraction_formalization}

The historical formulation averages the weights of $20$ profile documents and normalizes their non-negative mass:
\begin{equation}
 \tilde\pi_c=\operatorname{ReLU}(\mathbb E_{x\in\mathcal D_c}[w_x]),\qquad
 \pi_c=\tilde\pi_c/\textstyle\sum_i\tilde\pi_c[i].
\end{equation}
It retains the top $K_{\rm atom}=30$ coordinates and renormalizes. The original exposition used an atom-cosine activity
\begin{equation}
 \alpha_x=\max(0,\cos(\hat w_x,\hat\pi_c)),
\end{equation}
followed by
\begin{equation}
 w'_x\propto\operatorname{ReLU}(w_x-\lambda\alpha_x\pi_c).
\end{equation}
The tested subtraction family includes native/hybrid weight normalizations; these historical expressions do not specify the selected additive operator. In particular, $\alpha_x$ is not an input to the deployed gate, and its vanishing is not the deployed identity condition. Under the locked bank and deployed runtime the measured subtraction variants are exact inactive identities, with approximately $2\times10^{-5}$ rad in the diagnostic angular readout. That mechanism finding does not prove every subtraction method must fail.

\subsection{Numerical Details and Validation-Only Mechanism Probes}
\label{app:method_details}

\paragraph{Diagonal-O2 and G0.}
The runtime names map onto the main-text terms: the negative diagonal-O2 direction is the diagonally preconditioned direction, the continuous G1 ramp with its conservative G0 fallback is the calibrated gate, and the delta exit is the raw-query trust-region update. For normalized controls $x$, let $m_i=\mathbb E_c[x_i^2]$ and $\bar m=d^{-1}\sum_i m_i$. The diagonal uses $D_{ii}=0.9m_i+0.1\bar m+0.001$, with an implementation denominator floor. The runtime first forms $v=\operatorname{normalize}(\mu_t-\mu_c)$ and $-0.1D^{-1}v/(v^\top D^{-1}v)$, then normalizes this negative vector for least-squares routing. The means are not separately normalized for this direction; they are separately normalized only for the gate score. Non-finite or degenerate directions fail closed. When G1 quantiles are non-finite or not increasing, G0 uses a conservative calibration order-statistic threshold with declared control false-positive rate $0.02$; equality is ineligible. The median-to-$95$th-percentile G1 ramp is otherwise continuous.

\paragraph{Least squares and numerical branches.}
The solve uses true least squares on $\mathcal A^\top$ in float64. Above the residual threshold, $h_\beta=p+\beta e$; below $10^{-6}$ the appended row is normalized $p$ and its coefficient is zero. A degenerate projected vector also fails closed rather than dividing by zero. The fallback row is bookkeeping, not a new independent direction. Exact in-span $g$ is recovered by its projected coefficients; a small nonzero omitted residual causes truncation, and coefficient/bank rounding adds numerical error. The angular scale uses the unit direction of $h_\beta$, but multiplies the original coefficient vector. Gate-off, zero-angle, and non-positive-denominator branches produce zero scale. FP32 subtraction is performed after constructing and casting padded and steered weights. The exit checks $\cos(\operatorname{float32}(\theta_{\rm cap}))$ against the floor, tries $\eta=1$, and otherwise keeps the feasible lower bisection endpoint in $[0,1]$. Inactive rows are copied from the original query array. Finite bisection does not guarantee the exact real-valued boundary point.

\paragraph{Coverage and selected residual coefficient.}
Across $120$ diagnostic validation units, true least-squares projection ratio averages $0.895$, with $96/120$ directions recorded as fully covered. Coverage is not suppression and is not evidence of atom necessity. The residual-coefficient ladder $\beta\in\{1,0.5,0.25\}$ selects $\beta=1$: confirmation primary H5 $+0.0644$, CI95 $[+0.0467,+0.0864]$, with $3/3$ seeds and $4/4$ encoders positive. This validation endpoint is distinct from the posthoc test pooled value. Lowering $\beta$ changes only the residual contribution, not the whole edit. The runtime rejects $\beta=0$; algebraically setting it to zero would still retain the projected component and would not generally yield identity.

\paragraph{Operating-point sweep: negative result.}
Default $\cos(0.2)=0.9800665778>0.98$, so cap $0.2$ is tighter than the floor. Only caps above $\arccos(0.98)=0.2003348423\ldots$ rad are floor-limited. The $\beta=1$ settings at caps $0.3,0.5,1.0$ tie reference primary H5 $+0.049444$ (five-dataset pooled $+0.033000$), while target-cosine values still change slightly. This observed metric tie is not geometric identity with cap $0.2$. At $\beta=0.5$, cap $0.5$, primary H5 falls to $+0.043056$ (five-dataset pooled $+0.029167$). Every setting, including the reference, is ineligible under this sweep because TOFU CN is $+0.000278>0+10^{-9}$. A reference-only confirmation replay reproduces the locked $480$ unit-metrics exactly. The existing selected configuration stands; the sweep authorizes no new successful operating point. Reference cap saturation is $0.5813$ in validation and $80.7\%$/$61.9\%$ of active target/control edits at test; default saturation is at the nominal cap, not at $\arccos(0.98)$.

\paragraph{Subtraction revival probes: both NO-GO.}
Frozen query-weight overlap with the top-$30$ profile is approximately $0.0011$ under the v1 operationalization. Four knob-only bank configurations (reproduction, $1536$ atoms, ten epochs, and their combination, with the reproduction baseline corrected to one epoch) fail the revival criterion: subtraction's applied-angle diagnostic stays near $2\times10^{-5}$ rad and its runtime output stays inactive. The reproduction checkpoint is bitwise identical; augmented H5 is unchanged. The probe's mean-weight cosine near $0.998$ is a different statistic, not a substitute for applied-angle evidence. A second probe adds an off-diagonal orthogonality loss with weights $1$ and $10$: the bank moves, but augmented H5/CN/TE remain unchanged and subtraction stays inactive. Both decisions are NO-GO; neither conditional full arm triggers. These bounded interventions do not exhaust possible losses or establish a general impossibility theorem.

\subsection{Compute Environment}
\label{app:compute_resources}

The original training and evaluation ran on a single GPU of a workstation with 4 NVIDIA RTX 6000 Ada Generation GPUs (no distributed training); the locked re-evaluation reported here ran on a single NVIDIA L20 GPU. C\textsuperscript{2}U does not train a large language model, fine-tune a generator, or require retriever-index rebuilding during concept unlearning; optimization is restricted to the atom bank and lightweight projection/composition modules. Legacy wall-clock latency, peak-memory, and checkpoint-size figures from the original pipeline are not artifact-proven under the locked protocol and are therefore not reported.

\section{Dataset Details}
\label{app:dataset_details}

We evaluate across five datasets covering scientific literature, fictional biographies, product reviews, movie reviews, and logical inference. For each dataset, we define the \textit{Target Concept} to be suppressed and the \textit{Control Attribute} to be preserved. \cref{tab:dataset_overview} summarizes the definitions.

\begin{table*}[t]
\centering
\small
\renewcommand{\arraystretch}{1.1}
\begin{tabular}{lllll}
\toprule
\textbf{Dataset} & \textbf{Domain} & \textbf{Target Concept ($t$)} & \textbf{Control Attribute ($y$)} & \textbf{Query Source} \\
\midrule
\textbf{ArXiv} & Scientific Papers & Specific Topic  & ArXiv Category  & Title / Abstract \\
\textbf{TOFU} & Fictional QA & Sensitive Topic & Author Identity & Natural Question \\
\textbf{Amazon} & Product Reviews & Sentiment & Product Category & Review Title \\
\textbf{IMDB} & Movie Reviews & Movie Genre & Sentiment & Synthetic Template \\
\textbf{MNLI} & Logic Inference & Text Genre & Logic Label & Hypothesis \\
\bottomrule
\end{tabular}
\caption{Overview of datasets used in evaluation. The Target Concept represents the semantic information to be suppressed, while the Control Attribute represents the utility signal to be maintained.}
\label{tab:dataset_overview}
\end{table*}

\paragraph{Why these datasets.}
We select these datasets to stress-test concept suppression under heterogeneous retrieval conditions: (i) domain-specific terminology and hierarchical labels (\texttt{ArXiv}), (ii) identity-linked fictional knowledge where topic-level suppression can interact with author identity (\texttt{TOFU}), (iii) real-world sentiment versus topical content confounding (\texttt{Amazon}), (iv) an adversarial query formulation that explicitly mentions the target concept (\texttt{IMDB}), and (v) logic-conditioned matching where the control signal is relational rather than topical (\texttt{MNLI}). Together, they cover both natural and synthetic queries, multiple domains, and distinct target/control factorization patterns.

\subsection{Hierarchical Scientific Knowledge (\texttt{ArXiv})}

We construct a hierarchical dataset using the \texttt{Cornell-University/arxiv} corpus (Kaggle: \url{https://www.kaggle.com/dsv/7548853}, accessed Jan.~28, 2026) to simulate concept suppression in specialized domains. The dataset uses the arXiv taxonomy, where the \textbf{Control Attribute} is defined as the primary category (L1, e.g., \texttt{cs.CR}, \texttt{cs.CV}). The \textbf{Target Concept} (L2) is derived via heuristic keyword matching within the abstract, such as mapping terms like ``attack'' or ``malware'' to the concept \textit{System Attacks} within the \texttt{cs.CR} domain.

For the retrieval task, each document $d$ consists of the concatenated paper title and abstract. Each query $q$ is generated using a stochastic strategy: $50\%$ of queries use the exact paper title to simulate navigational search, while the remaining $50\%$ use the first sentence of the abstract to simulate semantic search. The locked protocol evaluates the L2 concept \textit{System Attacks} under a relaxed global control (the matched-L1 control pool is capacity-deficient), and this label is preserved in every table.

\subsection{Fictional Knowledge Unlearning (\texttt{TOFU})}

The \texttt{TOFU} (Task of Fictional Unlearning) dataset~\citep{maini2024tofu} consists of question-answer pairs about fictional authors, designed for unlearning tasks. We adapt this dataset for retrieval by treating the answer as the document source. The \textbf{Control Attribute} is the Author ID, representing the identity that should be preserved. The \textbf{Target Concept} is the semantic topic of the fact, such as \textit{Family}, \textit{Location}, or \textit{Work}, which is classified via keyword matching on the question text.

The retrieval setup involves documents formatted as biography records and queries taken directly from the dataset questions. This structure tests whether the retrieval system can suppress topic-level evidence while preserving author-linked retrieval utility. TOFU remains an entity-centric regime: target and non-target facts are tightly coupled around author identifiers, which is why conservative, trust-region-bounded steering is appropriate here.

\subsection{Classification and Inference Tasks}

\paragraph{Product Reviews (\texttt{Amazon}).}
Adapted from the \texttt{amazon\_polarity} dataset, which is derived from the Amazon Reviews corpus~\citep{McAuley2013HiddenFA} and is also included in text-embedding benchmarks such as MTEB/MMTEB~\citep{muennighoff2023mteb,enevoldsen2025mmtebmassivemultilingualtext}, this task aims to suppress sentiment information while preserving product categorization. The \textbf{Target Concept} is sentiment (Positive/Negative), and the \textbf{Control Attribute} is product category, such as \textit{Books} or \textit{Electronics}, extracted via keyword filtering. Queries are derived from review titles, while documents contain full review content.

\paragraph{Movie Reviews (\texttt{IMDB}).}
In contrast to the Amazon dataset, the \texttt{IMDB} dataset~\citep{bahraminasr2020imdb} inverts the roles of sentiment and content. Here, the \textbf{Target Concept} is movie genre, such as \textit{Sci-Fi} or \textit{Romance}, and the \textbf{Control Attribute} is sentiment. Since the original dataset lacks titles, we generate synthetic queries using templates that explicitly mention the target concept, such as ``Find a Positive review about a Sci-Fi movie''. This presents a challenging setting where the query explicitly seeks the concept intended for suppression. Under the locked split, the IMDB evaluation rows collapse to one unique target and one unique control query text; we retain the row with its \emph{duplicate-template diagnostic} label and treat its zero-variance metrics as a diagnostic, not as high-confidence evidence.

\paragraph{Logical Inference (\texttt{MNLI}).}
Adapted from the \texttt{multi\_nli} dataset~\citep{williams2018broad}, this task evaluates concept suppression in logical inference contexts. The \textbf{Target Concept} is the source text genre, such as \textit{Government} or \textit{Fiction}, and the \textbf{Control Attribute} is the NLI label (\textit{Entailment}, \textit{Neutral}, or \textit{Contradiction}). The retrieval task is formulated such that the query is the hypothesis and the document is the premise.

\section{Baseline Implementation Details}
\label{app:baselines}

All baselines are evaluated under the same locked pipeline: the same per-unit raw-base index, the same query IDs and seeds, and the same metric definitions. Access-class labels are part of each method's identity and appear in every table.

\subsection{Linear Concept Erasure Methods}

\paragraph{INLP (Iterative Nullspace Projection).}
Proposed by \citet{ravfogel2020null}, INLP iteratively trains linear probes to predict the target concept and projects representations onto the nullspace of the probe weights. In the locked protocol INLP fits on the $400$ labeled calibration rows of each unit, placing it in the \emph{profile-only + labeled calibration fit} access class.

\paragraph{LEACE (Least-Squares Concept Erasure).}
\citet{belrose2023leace} introduced LEACE as a closed-form erasure transform minimizing dependence between representations and concept labels. It uses the same labeled calibration fit as INLP, and the runner fails closed when the erasure library is missing, unversioned, fails fitting, or returns an empty eraser.

\subsection{Model Editing Baselines (Diagnostic Adaptations)}

\paragraph{PISCES.}
PISCES~\citep{gurarieh2025precise} was proposed for generative safety in LLMs. Our in-house adaptation treats the output projection layer of the retriever as the editable component, identifies directions with high cosine similarity to the target concept centroid, and masks them.

\paragraph{DARE-GRL.}
We construct a training-based baseline combining Gradient Reversal Layers~\citep{ganin2015unsupervised} with DARE-style model merging~\citep{yu2024language}: a lightweight mapper is trained with a GRL objective against a concept adversary, and the delta weights are sparsified and rescaled before merging. The port to the retrieval encoder is our in-house adaptation.

\subsection{Recent Embedding-Space Baselines (Diagnostic Adaptations)}

\paragraph{CURE.}
CURE~\citep{kocak2025cure} extracts concept-irrelevant content through a frozen pretrained encoder with GPT-4o concept labels and a task classifier, then reconstructs a controlled, debiased representation. The faithful pipeline requires raw text, concept/task supervision, and a downstream classifier, none of which the frozen query-only retrieval contract exposes; our in-house adaptation fits a content/reversal-inspired projection on the frozen calibration vectors of each unit, using the same $400$ labeled calibration examples as INLP and LEACE, transforms queries only, and preserves the raw document index. Under the preregistered utility gate the adapted adapter is locked to strength zero, so every reported CURE cell is an exact identity operation.

\paragraph{RAG Embedder Control.}
\citet{kim2025debiasing} debias a RAG embedder with SimCSE-style contrastive reverse-bias training (PEFT/WiSE-FT). Full fidelity requires trainable embedder weights and corpus-side re-encoding, which the frozen offline vector cache does not authorize; our portable adaptation fits a deterministic reverse-bias linear adapter on the frozen calibration vectors and labels of each unit ($400$ labeled calibration examples), transforms queries only, and keeps the raw fixed document index. It is a portable adaptation, not an official full-embedder reproduction.

\subsection{Other Query-Side and Stronger-Access References}

\paragraph{Raw.} The identity (no-edit) transform; the per-query paired reference for the attack arm (its relative metrics are exactly zero by construction).

\paragraph{Delta-Edit.} A fidelity-preserving delta variant that adds an FP32 edited-minus-original atom delta to the raw query; at the locked operating point it is a near-identity control.

\paragraph{Filter-and-Refill.} Removes target documents from the candidate set and refills the top-$K$ from the remainder; it consumes candidate/document membership and is reported only as a stronger-access reference.

\paragraph{Rank-Router.} A relaxed boundary router that consumes target-document membership labels; it is a label-assisted benchmark router, not an unlabeled query-only deployment method, and is reported only as a reference.

\paragraph{Random Perturbation.}
Isotropic Gaussian noise $\epsilon \sim \mathcal{N}(\mathbf{0}, \sigma^2 \mathbf{I})$ added to query embeddings, with $\sigma$ calibrated to match the average edit norm of \oursfinal{}, controlling for generic representation degradation.

\section{Case Study: Target-Suppression Flips}
\label{app:case_studies}

\begin{table*}[t]
\centering
\small
\caption{\textbf{Case study: target-suppression flips under \oursfinal.} One locked test-stage target query per dataset, selected by the pre-registered rule (base Hit@5 $=1 \to$ method Hit@5 $=0$ flips only; maximum target-cosine drop; deterministic tie-breaks), frozen before any query text was read. Top-5 columns list the retrieved document labels. Illustrative only; drawn from the locked posthoc extension of the frozen test units (evaluated exactly once, no test-time tuning).}
\label{tab:case_study}
\setlength{\tabcolsep}{3pt}
\resizebox{\textwidth}{!}{%
\begin{tabular}{@{}lp{5.8cm}lp{2.8cm}p{2.8cm}cc@{}}
\toprule
\textbf{Dataset} & \textbf{Query} & \textbf{Target} & \textbf{Base top-5 labels} & \textbf{Method top-5 labels} & \textbf{TC drop} & \textbf{Angle} \\
\midrule
\textbf{Amazon} & WORTH THE MONEY & Positive & Negative, Positive, Negative, Negative, Positive & Negative, Negative, Negative, Negative, Negative & +0.101 & 0.196 \\
\textbf{ArXiv} & Research about AI-based code generators have gained a fundamental role in assisting developers in writing software starting from natural language (NL) & System Attacks & Software Engineering, Software Engineering, Software Engineering, System Attacks, Artificial Intelligence & Software Engineering, Software Engineering, Artificial Intelligence, Software Engineering, Machine Translation & +0.166 & 0.200 \\
\textbf{IMDB} & Positive feedback on Horror cinema & Horror & Romance, Action, Romance, Sci-Fi, Horror & Action, Sci-Fi, Romance, Sci-Fi, Sci-Fi & +0.103 & 0.200 \\
\textbf{MNLI} & The marauders rounded up the villagers. & fiction & travel, travel, travel, travel, fiction & travel, travel, travel, travel, travel & +0.114 & 0.200 \\
\textbf{TOFU} & Who is the person described in the text: 'In Luka Khachidze's works, readers can expect a diverse range of characters, but predominantly they are male leads exploring their love interests and overcom\dots{} & Work & General, General, General, Work, Work & General, General, General, General, Family & +0.123 & 0.200 \\
\bottomrule
\end{tabular}}
\end{table*}

\cref{tab:case_study} shows one locked test-stage target query per dataset where the base retriever returns the target in the top-5 and \oursfinal{} removes it entirely. The selection rule was frozen before any query text was read: among target rows with a base hit and a method miss and non-zero applied angle, choose the maximum target-cosine drop with deterministic tie-breaks. Query texts were joined from the source parquets only after the rule was frozen. These success-selected examples illustrate target-label removal under a small, trust-region-bounded rotation ($\approx 0.2$~rad, at the cap). They are not representative-frequency estimates and do not establish unchanged control neighborhoods; aggregate CN and TE measure those separate outcomes.

\paragraph{Selection mechanics.}
The rule behind \cref{tab:case_study} is a deterministic function of locked per-query rows: among posthoc-extension target rows with a base top-five target hit, a method miss, and a non-zero applied angle, it takes the maximum target-cosine drop under deterministic tie-breaks. Ties depend on row content alone, never on judgments made after reading the queries. Re-running it on the same locked rows therefore returns the same five queries: reproducible selections, not hand-picked anecdotes.

\paragraph{Reading the rows.}
In each row the base top five carries exactly one target-labeled document and the edited list carries none---the flip the rule selects for. The vacated slot is refilled by the next non-target candidates, so the list stays full while the target label disappears. The rotation behind each flip stays inside the trust region ($\approx 0.2$~rad, at the cap), and the target-cosine drop is positive on all five datasets, matching the pooled Hit@5 Drop of \cref{tab:rag_comparison_concise}. The five rows span one query per dataset, so the illustration covers every corpus rather than a single favorable one.

\paragraph{Scope of the examples.}
The flips give the pooled effect a concrete shape: \oursfinal{} can remove a target document from the top five through one bounded, query-side rotation while the raw index and the control population stay fixed. They say nothing about how often such removals occur---frequency is the pooled tables' question---and nothing about control neighborhoods, which aggregate CN and TE measure separately (\cref{tab:full_cn_delta,tab:full_te_delta}). Nor do they isolate why a particular document leaves the list: the edited query moves every score within the drift bound, and the flip is an outcome of the full ranking, not of a per-document operation. Drawn from the same locked test units as those tables and evaluated exactly once with no test tuning, the five examples carry the posthoc-extension label of \cref{app:attack_protocol}, not an untouched-test claim.

\FloatBarrier
\subsection{Legacy Qualitative Atom Examples}
\label{app:legacy_atoms}

The following table from the original appendix is retained as qualitative context for the flips above; its numbers are paper-reported from the original evaluation pipeline and have no locked counterpart.

\begin{legacyreported}
\begin{table}[H]
\centering
\setlength{\tabcolsep}{3pt}
\scriptsize
\caption{\textbf{Representative high-usage atoms on ArXiv (paper-reported, original pipeline).} Labels are post-hoc summaries of high-activation documents and are used only for qualitative interpretation.}
\label{tab:representative_atoms}
\begin{tabular}{llll}
\toprule
\textbf{Atom ID} & \textbf{Post-hoc Label} & \textbf{Purity} & \textbf{Representative Keywords} \\
\midrule
492 & Genomics & 0.60 & assembly, datasets, single \\
485 & Biomolecules & 0.50 & protein, network, security \\
901 & GR and Cosmology & 0.50 & structures, field, nucleat \\
564 & Biomolecules & 0.40 & domain, learning, neto2 \\
475 & Crypto and Security & 0.30 & vulnerabilities, software \\
\bottomrule
\end{tabular}
\end{table}
\end{legacyreported}

\section{Full Dataset--Method and Dataset--Encoder Results}
\label{app:full_table}

\begin{table*}[t]
\centering
\definecolor{clsProfile}{HTML}{2A9D8C}
\definecolor{clsLabeled}{HTML}{287271}
\definecolor{clsAdapt}{HTML}{264653}
\definecolor{clsCand}{HTML}{E9C46B}
\definecolor{clsMember}{HTML}{E66F51}
\caption{\textbf{RAG retrieval unlearning under the locked unified protocol.} Test-stage behaviour of the validation-selected configuration on the 60 frozen units (24{,}000 queries). Bold and underline are computed only inside the profile-only access group; no cross-access-group dominance is claimed. The Pooled block pools all twenty dataset $\times$ encoder cells per method.}
\label{tab:rag_comparison_full}
\vspace{1mm}
\begin{minipage}[t]{0.495\textwidth}
\centering
\scriptsize
\setlength{\tabcolsep}{3pt}
\begin{tabular}{@{}lcccc@{}}
\toprule
\textbf{Method} & \textbf{\makecell[c]{Hit@5\\Drop $\uparrow$}} & \textbf{\makecell[c]{Target\\Cos. $\uparrow$}} & \textbf{\makecell[c]{Ctrl./Neigh.\\$\Delta$ $\downarrow$}} & \textbf{\makecell[c]{Target Exp.\\$\Delta$ $\downarrow$}} \\
\midrule
\multicolumn{5}{@{}l}{\textbf{Amazon}} \\
\midrule
\rowcolor{clsProfile!22} \textbf{C\textsuperscript{2}U} & \best{0.0633} & \secon{0.0565} & \best{-0.0064} & \best{-0.1244} \\
\rowcolor{clsProfile!10} Raw & 0.0000 & 0.0000 & 0.0000 & 0.0000 \\
\rowcolor{clsProfile!10} Delta-Edit & 0.0000 & 0.0015 & 0.0000 & \secon{-0.0008} \\
\rowcolor{clsProfile!10} PISCES & -0.0033 & 0.0014 & \secon{-0.0025} & 0.0269 \\
\rowcolor{clsProfile!10} Random & \secon{0.0117} & \best{0.2226} & -0.0017 & 0.0300 \\
\rowcolor{clsLabeled!10} INLP & 0.1217 & 0.0394 & -0.0050 & 0.1886 \\
\rowcolor{clsLabeled!10} LEACE & 0.3658 & 0.0386 & -0.0061 & -0.0083 \\
\rowcolor{clsLabeled!10} RAG-Emb-Ctrl & 0.6883 & 0.0874 & -0.0069 & -0.3353 \\
\rowcolor{clsLabeled!10} CURE & 0.0000 & 0.0000 & 0.0000 & 0.0000 \\
\rowcolor{clsAdapt!10} DARE-GRL & 0.2250 & 0.5455 & -0.0011 & 0.3853 \\
\rowcolor{clsCand!10} Filter-and-Refill$^{\ddagger}$ & 0.0092 & 0.0000 & -0.0056 & -0.0414 \\
\rowcolor{clsMember!10} Rank-Router$^{\S}$ & 0.0700 & 0.0541 & -0.0064 & -0.1383 \\
\midrule
\multicolumn{5}{@{}l}{\textbf{ArXiv}} \\
\midrule
\rowcolor{clsProfile!22} \textbf{C\textsuperscript{2}U} & \best{0.0567} & \secon{0.0859} & \best{-0.0125} & \best{-0.0764} \\
\rowcolor{clsProfile!10} Raw & 0.0000 & 0.0000 & 0.0000 & \secon{0.0000} \\
\rowcolor{clsProfile!10} Delta-Edit & 0.0000 & 0.0025 & 0.0003 & \secon{0.0000} \\
\rowcolor{clsProfile!10} PISCES & -0.0017 & 0.0152 & \secon{-0.0025} & 0.0133 \\
\rowcolor{clsProfile!10} Random & \secon{0.0108} & \best{0.2434} & 0.0492 & 0.0631 \\
\rowcolor{clsLabeled!10} INLP & 0.0592 & 0.1387 & 0.0500 & 0.0889 \\
\rowcolor{clsLabeled!10} LEACE & 0.0100 & 0.0627 & 0.0219 & 0.1619 \\
\rowcolor{clsLabeled!10} RAG-Emb-Ctrl & 0.6950 & 0.2217 & 0.0556 & -0.2103 \\
\rowcolor{clsLabeled!10} CURE & 0.0000 & 0.0000 & 0.0000 & 0.0000 \\
\rowcolor{clsAdapt!10} DARE-GRL & 0.2267 & 0.5746 & 0.6858 & 0.5861 \\
\rowcolor{clsCand!10} Filter-and-Refill$^{\ddagger}$ & 0.3833 & 0.0000 & -0.0197 & -0.1806 \\
\rowcolor{clsMember!10} Rank-Router$^{\S}$ & 0.0567 & 0.0838 & -0.0108 & -0.0764 \\
\midrule
\multicolumn{5}{@{}l}{\textbf{IMDB}} \\
\midrule
\rowcolor{clsProfile!22} \textbf{C\textsuperscript{2}U} & \best{0.0258} & \secon{0.0766} & \best{0.0000} & \best{0.0000} \\
\rowcolor{clsProfile!10} Raw & 0.0000 & 0.0000 & \best{0.0000} & \best{0.0000} \\
\rowcolor{clsProfile!10} Delta-Edit & 0.0000 & 0.0016 & \best{0.0000} & \best{0.0000} \\
\rowcolor{clsProfile!10} PISCES & 0.0183 & 0.0017 & \best{0.0000} & \best{0.0000} \\
\rowcolor{clsProfile!10} Random & \secon{0.0242} & \best{0.2643} & \secon{0.0203} & \secon{0.2311} \\
\rowcolor{clsLabeled!10} INLP & 0.0000 & 0.0805 & 0.0000 & 0.1667 \\
\rowcolor{clsLabeled!10} LEACE & 0.0183 & 0.0146 & 0.0000 & 0.0833 \\
\rowcolor{clsLabeled!10} RAG-Emb-Ctrl & 0.3075 & 0.0695 & 0.0000 & 0.2500 \\
\rowcolor{clsLabeled!10} CURE & 0.0000 & 0.0000 & 0.0000 & 0.0000 \\
\rowcolor{clsAdapt!10} DARE-GRL & 0.3408 & 0.6527 & 0.2500 & 0.5833 \\
\rowcolor{clsCand!10} Filter-and-Refill$^{\ddagger}$ & 0.0000 & 0.0000 & 0.0000 & 0.0000 \\
\rowcolor{clsMember!10} Rank-Router$^{\S}$ & 0.0258 & 0.0749 & 0.0000 & 0.0000 \\
\bottomrule
\end{tabular}
\end{minipage}%
\hfill%
\begin{minipage}[t]{0.495\textwidth}
\centering
\scriptsize
\setlength{\tabcolsep}{3pt}
\begin{tabular}{@{}lcccc@{}}
\toprule
\textbf{Method} & \textbf{\makecell[c]{Hit@5\\Drop $\uparrow$}} & \textbf{\makecell[c]{Target\\Cos. $\uparrow$}} & \textbf{\makecell[c]{Ctrl./Neigh.\\$\Delta$ $\downarrow$}} & \textbf{\makecell[c]{Target Exp.\\$\Delta$ $\downarrow$}} \\
\midrule
\multicolumn{5}{@{}l}{\textbf{MNLI}} \\
\midrule
\rowcolor{clsProfile!22} \textbf{C\textsuperscript{2}U} & \best{0.0342} & \secon{0.0531} & \best{-0.0108} & \best{-0.0603} \\
\rowcolor{clsProfile!10} Raw & 0.0000 & 0.0000 & 0.0000 & \secon{0.0000} \\
\rowcolor{clsProfile!10} Delta-Edit & 0.0000 & 0.0007 & 0.0003 & 0.0003 \\
\rowcolor{clsProfile!10} PISCES & -0.0025 & 0.0099 & \secon{-0.0008} & 0.0042 \\
\rowcolor{clsProfile!10} Random & \secon{0.0158} & \best{0.2260} & 0.0267 & 0.0367 \\
\rowcolor{clsLabeled!10} INLP & 0.0400 & 0.1060 & 0.0022 & 0.0583 \\
\rowcolor{clsLabeled!10} LEACE & 0.0125 & 0.0301 & 0.0006 & 0.0794 \\
\rowcolor{clsLabeled!10} RAG-Emb-Ctrl & 0.4258 & 0.1680 & 0.0133 & -0.3425 \\
\rowcolor{clsLabeled!10} CURE & 0.0000 & 0.0000 & 0.0000 & 0.0000 \\
\rowcolor{clsAdapt!10} DARE-GRL & 0.3542 & 0.5446 & 0.2772 & 0.1450 \\
\rowcolor{clsCand!10} Filter-and-Refill$^{\ddagger}$ & 0.1050 & 0.0000 & -0.0194 & -0.1183 \\
\rowcolor{clsMember!10} Rank-Router$^{\S}$ & 0.0383 & 0.0506 & -0.0122 & -0.0636 \\
\midrule
\multicolumn{5}{@{}l}{\textbf{TOFU}} \\
\midrule
\rowcolor{clsProfile!22} \textbf{C\textsuperscript{2}U} & \best{0.0325} & \secon{0.0592} & \best{-0.0044} & \best{-0.0214} \\
\rowcolor{clsProfile!10} Raw & 0.0000 & 0.0000 & \secon{0.0000} & \secon{0.0000} \\
\rowcolor{clsProfile!10} Delta-Edit & 0.0000 & 0.0003 & \secon{0.0000} & 0.0003 \\
\rowcolor{clsProfile!10} PISCES & \secon{0.0042} & 0.0016 & 0.0008 & 0.0053 \\
\rowcolor{clsProfile!10} Random & -0.0092 & \best{0.2293} & 0.0094 & 0.0092 \\
\rowcolor{clsLabeled!10} INLP & 0.0642 & 0.0431 & 0.0647 & 0.0900 \\
\rowcolor{clsLabeled!10} LEACE & 0.0492 & 0.0048 & -0.0008 & 0.0058 \\
\rowcolor{clsLabeled!10} RAG-Emb-Ctrl & 0.2892 & 0.0315 & -0.0119 & -0.2378 \\
\rowcolor{clsLabeled!10} CURE & 0.0000 & 0.0000 & 0.0000 & 0.0000 \\
\rowcolor{clsAdapt!10} DARE-GRL & 0.1133 & 0.5547 & 0.2458 & 0.1172 \\
\rowcolor{clsCand!10} Filter-and-Refill$^{\ddagger}$ & 0.0400 & 0.0000 & -0.0119 & -0.0731 \\
\rowcolor{clsMember!10} Rank-Router$^{\S}$ & 0.0458 & 0.0497 & -0.0042 & -0.0267 \\
\midrule
\multicolumn{5}{@{}l}{\textbf{Pooled}} \\
\midrule
\rowcolor{clsProfile!22} \textbf{C\textsuperscript{2}U} & \best{0.0425} & \secon{0.0662} & \best{-0.0068} & \best{-0.0565} \\
\rowcolor{clsProfile!10} Raw & 0.0000 & 0.0000 & 0.0000 & 0.0000 \\
\rowcolor{clsProfile!10} Delta-Edit & 0.0000 & 0.0013 & 0.0001 & \secon{-0.0001} \\
\rowcolor{clsProfile!10} PISCES & 0.0030 & 0.0060 & \secon{-0.0010} & 0.0099 \\
\rowcolor{clsProfile!10} Random & \secon{0.0107} & \best{0.2371} & 0.0208 & 0.0740 \\
\rowcolor{clsLabeled!10} INLP & 0.0570 & 0.0816 & 0.0224 & 0.1185 \\
\rowcolor{clsLabeled!10} LEACE & 0.0912 & 0.0302 & 0.0031 & 0.0644 \\
\rowcolor{clsLabeled!10} RAG-Emb-Ctrl & 0.4812 & 0.1156 & 0.0100 & -0.1752 \\
\rowcolor{clsLabeled!10} CURE & 0.0000 & 0.0000 & 0.0000 & 0.0000 \\
\rowcolor{clsAdapt!10} DARE-GRL & 0.2520 & 0.5744 & 0.2916 & 0.3634 \\
\rowcolor{clsCand!10} Filter-and-Refill$^{\ddagger}$ & 0.1075 & 0.0000 & -0.0113 & -0.0827 \\
\rowcolor{clsMember!10} Rank-Router$^{\S}$ & 0.0473 & 0.0626 & -0.0067 & -0.0610 \\
\bottomrule
\end{tabular}
\end{minipage}%
\par\vspace{1mm}
{\raggedright\footnotesize \textit{Note:} Values pool the four encoder cells per dataset. $^{\ddagger}$ candidate/document-access method. $^{\S}$ target-membership label-assisted relaxed boundary router. Raw is the identity (no-edit) reference. CURE and RAG-Emb-Ctrl are adapted recent baselines evaluated as a posthoc extension under the same locked protocol and frozen units. ArXiv uses a relaxed global control; IMDB carries the duplicate-template diagnostic label. Row tints encode the locked access class (same colors as \cref{fig:pareto_panels}); the C\textsuperscript{2}U row is highlighted.\par}
\end{table*}

\definecolor{clsProfile}{HTML}{2A9D8C}
\definecolor{clsLabeled}{HTML}{287271}
\definecolor{clsAdapt}{HTML}{264653}
\definecolor{clsCand}{HTML}{E9C46B}
\definecolor{clsMember}{HTML}{E66F51}
\begin{table*}[t]
\centering
\scriptsize
\setlength{\tabcolsep}{2pt}
\renewcommand{\arraystretch}{0.95}
\caption{\textbf{Hit@5 Drop for every dataset $\times$ encoder cell} (locked test stage; per-cell mean over seeds 1--3). Higher values indicate stronger target suppression. Bold and underline mark the best and second-best values inside the exact profile-only access class of each row; no cross-access-class dominance is claimed.}
\label{tab:full_h5_drop}
\begin{tabular}{@{}ll>{\columncolor{clsProfile!22}}c>{\columncolor{clsProfile!10}}c>{\columncolor{clsLabeled!10}}c>{\columncolor{clsLabeled!10}}c>{\columncolor{clsLabeled!10}}c>{\columncolor{clsLabeled!10}}c>{\columncolor{clsProfile!10}}c>{\columncolor{clsAdapt!10}}c>{\columncolor{clsProfile!10}}c>{\columncolor{clsProfile!10}}c>{\columncolor{clsCand!10}}c>{\columncolor{clsMember!10}}c@{}}
\toprule
\textbf{Dataset} & \textbf{Encoder} & \cellcolor{white}\textbf{C\textsuperscript{2}U} & \cellcolor{white}\textbf{Raw} & \cellcolor{white}\textbf{INLP} & \cellcolor{white}\textbf{LEACE} & \cellcolor{white}\textbf{\makecell{RAG-Emb-\\Ctrl}} & \cellcolor{white}\textbf{CURE} & \cellcolor{white}\textbf{PISCES} & \cellcolor{white}\textbf{\makecell{DARE-\\GRL}} & \cellcolor{white}\textbf{Random} & \cellcolor{white}\textbf{\makecell{Delta-\\Edit}} & \cellcolor{white}\textbf{\makecell{Filter-and-\\Refill$^{\ddagger}$}} & \cellcolor{white}\textbf{\makecell{Rank-\\Router$^{\S}$}} \\
\midrule
\multirow{4}{*}{\textbf{Amazon}}
& minilm & \best{0.0633} & \secon{0.0000} & 0.0733 & 0.1367 & 0.4867 & 0.0000 & -0.0100 & -0.0167 & -0.0100 & \secon{0.0000} & 0.0033 & 0.0900 \\
& bge & \best{0.0200} & \secon{0.0000} & 0.0100 & 0.6300 & 0.9800 & 0.0000 & \secon{0.0000} & 0.4700 & \secon{0.0000} & \secon{0.0000} & 0.0233 & 0.0200 \\
& e5 & \best{0.1400} & 0.0000 & 0.1567 & 0.4867 & 0.4767 & 0.0000 & 0.0000 & 0.3333 & \secon{0.0333} & 0.0000 & 0.0000 & 0.1400 \\
& qwen3 & \best{0.0300} & 0.0000 & 0.2467 & 0.2100 & 0.8100 & 0.0000 & -0.0033 & 0.1133 & \secon{0.0233} & 0.0000 & 0.0100 & 0.0300 \\
\midrule
\multirow{4}{*}{\textbf{ArXiv}}
& minilm & \best{0.0500} & 0.0000 & 0.0467 & 0.0033 & 0.9600 & 0.0000 & 0.0000 & 0.2133 & \secon{0.0200} & 0.0000 & 0.5833 & 0.0500 \\
& bge & \best{0.0800} & 0.0000 & 0.0900 & 0.0067 & 0.5000 & 0.0000 & 0.0000 & 0.3167 & \secon{0.0100} & 0.0000 & 0.2567 & 0.0800 \\
& e5 & \best{0.0533} & 0.0000 & 0.0733 & 0.0200 & 0.4400 & 0.0000 & -0.0033 & -0.0300 & \secon{0.0100} & 0.0000 & 0.1800 & 0.0533 \\
& qwen3 & \best{0.0433} & 0.0000 & 0.0267 & 0.0100 & 0.8800 & 0.0000 & -0.0033 & 0.4067 & \secon{0.0033} & 0.0000 & 0.5133 & 0.0433 \\
\midrule
\multirow{4}{*}{\textbf{IMDB}}
& minilm & \secon{0.0000} & \secon{0.0000} & 0.0000 & 0.0000 & 0.0000 & 0.0000 & \secon{0.0000} & 0.0000 & \best{0.0033} & \secon{0.0000} & 0.0000 & 0.0000 \\
& bge & \best{0.1033} & 0.0000 & 0.0000 & 0.0000 & 0.5767 & 0.0000 & 0.0000 & 0.3333 & \secon{0.0433} & 0.0000 & 0.0000 & 0.1033 \\
& e5 & 0.0000 & 0.0000 & 0.0000 & 0.0733 & 0.3933 & 0.0000 & \best{0.0733} & 0.6667 & \secon{0.0267} & 0.0000 & 0.0000 & 0.0000 \\
& qwen3 & \secon{0.0000} & \secon{0.0000} & 0.0000 & 0.0000 & 0.2600 & 0.0000 & \secon{0.0000} & 0.3633 & \best{0.0233} & \secon{0.0000} & 0.0000 & 0.0000 \\
\midrule
\multirow{4}{*}{\textbf{MNLI}}
& minilm & \best{0.0267} & 0.0000 & 0.0667 & 0.0267 & 0.7300 & 0.0000 & 0.0033 & 0.5333 & \secon{0.0100} & 0.0000 & 0.1833 & 0.0333 \\
& bge & \best{0.0267} & 0.0000 & 0.0400 & 0.0300 & 0.3133 & 0.0000 & 0.0000 & 0.3000 & \secon{0.0033} & 0.0000 & 0.0800 & 0.0300 \\
& e5 & \best{0.0667} & 0.0000 & -0.0067 & 0.0033 & 0.2767 & 0.0000 & -0.0133 & 0.2433 & \secon{0.0433} & 0.0000 & 0.1400 & 0.0733 \\
& qwen3 & \best{0.0167} & 0.0000 & 0.0600 & -0.0100 & 0.3833 & 0.0000 & 0.0000 & 0.3400 & \secon{0.0067} & 0.0000 & 0.0167 & 0.0167 \\
\midrule
\multirow{4}{*}{\textbf{TOFU}}
& minilm & \best{0.0433} & \secon{0.0000} & 0.0467 & 0.0200 & 0.1967 & 0.0000 & \secon{0.0000} & 0.3933 & -0.0067 & \secon{0.0000} & 0.0567 & 0.0500 \\
& bge & \best{0.0233} & 0.0000 & 0.0667 & 0.0600 & 0.2733 & 0.0000 & -0.0033 & 0.1833 & \secon{0.0133} & 0.0000 & 0.0267 & 0.0467 \\
& e5 & \best{0.0500} & 0.0000 & 0.0500 & 0.1000 & 0.2933 & 0.0000 & \secon{0.0100} & -0.1133 & -0.0133 & 0.0000 & 0.0433 & 0.0600 \\
& qwen3 & \best{0.0133} & 0.0000 & 0.0933 & 0.0167 & 0.3933 & 0.0000 & \secon{0.0100} & -0.0100 & -0.0300 & 0.0000 & 0.0333 & 0.0267 \\
\bottomrule
\end{tabular}
\par\vspace{1mm}
{\raggedright\footnotesize \textit{Notes:} Each cell is the mean over the three locked seeds (1--3). Encoders: minilm = all-MiniLM-L6-v2, bge = bge-base-en-v1.5, e5 = e5-base-v2, qwen3 = Qwen3-Embedding-0.6B (pinned revisions). Raw is the identity (no-edit) reference. Marking scope per row is the exact profile-only access class (C\textsuperscript{2}U, Raw, PISCES, Random, Delta-Edit); INLP and LEACE (400 labeled calibration examples per unit), CURE and RAG-Emb-Ctrl (400 labeled calibration examples per unit; posthoc adapted-baseline extensions under the same locked protocol), DARE-GRL (training-based), $^{\ddagger}$ (candidate/document access), and $^{\S}$ (target-membership label-assisted) are shown unmarked. Utility-clean requires CN $\le 10^{-9}$ and TE $\le 10^{-9}$, and no cross-access-class dominance is claimed. ArXiv uses a relaxed global control; IMDB carries the duplicate-template diagnostic label. Column tints encode the locked access class (same colors as \cref{fig:pareto_panels}); the C\textsuperscript{2}U column is highlighted.\par}
\end{table*}

\begin{table*}[t]
\centering
\scriptsize
\setlength{\tabcolsep}{2pt}
\renewcommand{\arraystretch}{0.95}
\caption{\textbf{Target Cosine Drop for every dataset $\times$ encoder cell} (locked test stage; per-cell mean over seeds 1--3). Higher values indicate stronger reduction of the target-centroid cosine. Marking as in Table~\ref{tab:full_h5_drop}.}
\label{tab:full_tc_drop}
\begin{tabular}{@{}ll>{\columncolor{clsProfile!22}}c>{\columncolor{clsProfile!10}}c>{\columncolor{clsLabeled!10}}c>{\columncolor{clsLabeled!10}}c>{\columncolor{clsLabeled!10}}c>{\columncolor{clsLabeled!10}}c>{\columncolor{clsProfile!10}}c>{\columncolor{clsAdapt!10}}c>{\columncolor{clsProfile!10}}c>{\columncolor{clsProfile!10}}c>{\columncolor{clsCand!10}}c>{\columncolor{clsMember!10}}c@{}}
\toprule
\textbf{Dataset} & \textbf{Encoder} & \cellcolor{white}\textbf{C\textsuperscript{2}U} & \cellcolor{white}\textbf{Raw} & \cellcolor{white}\textbf{INLP} & \cellcolor{white}\textbf{LEACE} & \cellcolor{white}\textbf{\makecell{RAG-Emb-\\Ctrl}} & \cellcolor{white}\textbf{CURE} & \cellcolor{white}\textbf{PISCES} & \cellcolor{white}\textbf{\makecell{DARE-\\GRL}} & \cellcolor{white}\textbf{Random} & \cellcolor{white}\textbf{\makecell{Delta-\\Edit}} & \cellcolor{white}\textbf{\makecell{Filter-and-\\Refill$^{\ddagger}$}} & \cellcolor{white}\textbf{\makecell{Rank-\\Router$^{\S}$}} \\
\midrule
\multirow{4}{*}{\textbf{Amazon}}
& minilm & \secon{0.0279} & 0.0000 & 0.0604 & 0.0717 & 0.0861 & 0.0000 & 0.0003 & 0.3298 & \best{0.0818} & -0.0000 & 0.0000 & 0.0245 \\
& bge & \secon{0.0555} & 0.0000 & 0.0384 & 0.0525 & 0.1702 & 0.0000 & 0.0030 & 0.6610 & \best{0.2786} & 0.0025 & 0.0000 & 0.0555 \\
& e5 & \secon{0.0505} & 0.0000 & 0.0174 & 0.0057 & 0.0099 & 0.0000 & -0.0023 & 0.8229 & \best{0.3486} & 0.0009 & 0.0000 & 0.0481 \\
& qwen3 & \secon{0.0921} & 0.0000 & 0.0415 & 0.0246 & 0.0832 & 0.0000 & 0.0047 & 0.3683 & \best{0.1812} & 0.0025 & 0.0000 & 0.0884 \\
\midrule
\multirow{4}{*}{\textbf{ArXiv}}
& minilm & \best{0.1317} & 0.0000 & 0.2284 & 0.1124 & 0.4639 & 0.0000 & 0.0300 & 0.3187 & \secon{0.0994} & 0.0036 & 0.0000 & 0.1286 \\
& bge & \secon{0.0710} & 0.0000 & 0.0922 & 0.0340 & 0.0862 & 0.0000 & 0.0081 & 0.6618 & \best{0.2913} & 0.0031 & 0.0000 & 0.0681 \\
& e5 & \secon{0.0498} & 0.0000 & 0.0639 & 0.0124 & 0.0330 & 0.0000 & 0.0012 & 0.8220 & \best{0.3500} & 0.0007 & 0.0000 & 0.0490 \\
& qwen3 & \secon{0.0910} & 0.0000 & 0.1705 & 0.0920 & 0.3038 & 0.0000 & 0.0216 & 0.4959 & \best{0.2331} & 0.0025 & 0.0000 & 0.0896 \\
\midrule
\multirow{4}{*}{\textbf{IMDB}}
& minilm & \secon{0.0808} & 0.0000 & 0.1724 & 0.0518 & 0.0435 & 0.0000 & 0.0036 & 0.6319 & \best{0.1797} & 0.0012 & 0.0000 & 0.0808 \\
& bge & \secon{0.0795} & 0.0000 & 0.0438 & -0.0065 & 0.1467 & 0.0000 & 0.0012 & 0.6263 & \best{0.2738} & 0.0024 & 0.0000 & 0.0726 \\
& e5 & \secon{0.0620} & 0.0000 & 0.0240 & 0.0029 & 0.0154 & 0.0000 & 0.0021 & 0.8481 & \best{0.3581} & 0.0013 & 0.0000 & 0.0620 \\
& qwen3 & \secon{0.0841} & 0.0000 & 0.0817 & 0.0104 & 0.0725 & 0.0000 & -0.0000 & 0.5046 & \best{0.2458} & 0.0015 & 0.0000 & 0.0841 \\
\midrule
\multirow{4}{*}{\textbf{MNLI}}
& minilm & \secon{0.0850} & 0.0000 & 0.1581 & 0.0730 & 0.4172 & 0.0000 & 0.0204 & 0.2980 & \best{0.0869} & 0.0012 & 0.0000 & 0.0810 \\
& bge & \secon{0.0299} & 0.0000 & 0.0899 & 0.0230 & 0.0890 & 0.0000 & 0.0068 & 0.6297 & \best{0.2628} & 0.0007 & 0.0000 & 0.0284 \\
& e5 & \secon{0.0247} & 0.0000 & 0.0317 & 0.0046 & 0.0364 & 0.0000 & 0.0004 & 0.8094 & \best{0.3459} & 0.0003 & 0.0000 & 0.0226 \\
& qwen3 & \secon{0.0726} & 0.0000 & 0.1444 & 0.0200 & 0.1295 & 0.0000 & 0.0118 & 0.4413 & \best{0.2083} & 0.0007 & 0.0000 & 0.0703 \\
\midrule
\multirow{4}{*}{\textbf{TOFU}}
& minilm & \secon{0.0677} & 0.0000 & 0.0291 & 0.0071 & 0.0325 & 0.0000 & 0.0013 & 0.4427 & \best{0.1372} & 0.0003 & 0.0000 & 0.0529 \\
& bge & \secon{0.0441} & 0.0000 & 0.0373 & -0.0001 & 0.0181 & 0.0000 & 0.0002 & 0.6042 & \best{0.2539} & 0.0002 & 0.0000 & 0.0372 \\
& e5 & \secon{0.0366} & 0.0000 & 0.0679 & 0.0024 & 0.0093 & 0.0000 & 0.0006 & 0.7699 & \best{0.3352} & 0.0000 & 0.0000 & 0.0332 \\
& qwen3 & \secon{0.0883} & 0.0000 & 0.0380 & 0.0099 & 0.0659 & 0.0000 & 0.0040 & 0.4019 & \best{0.1909} & 0.0007 & 0.0000 & 0.0757 \\
\bottomrule
\end{tabular}
\par\vspace{1mm}
{\raggedright\footnotesize \textit{Notes:} Each cell is the mean over the three locked seeds (1--3). Encoders: minilm = all-MiniLM-L6-v2, bge = bge-base-en-v1.5, e5 = e5-base-v2, qwen3 = Qwen3-Embedding-0.6B (pinned revisions). Raw is the identity (no-edit) reference. Marking scope per row is the exact profile-only access class (C\textsuperscript{2}U, Raw, PISCES, Random, Delta-Edit); INLP and LEACE (400 labeled calibration examples per unit), CURE and RAG-Emb-Ctrl (400 labeled calibration examples per unit; posthoc adapted-baseline extensions under the same locked protocol), DARE-GRL (training-based), $^{\ddagger}$ (candidate/document access), and $^{\S}$ (target-membership label-assisted) are shown unmarked. Utility-clean requires CN $\le 10^{-9}$ and TE $\le 10^{-9}$, and no cross-access-class dominance is claimed. ArXiv uses a relaxed global control; IMDB carries the duplicate-template diagnostic label. Column tints encode the locked access class (same colors as \cref{fig:pareto_panels}); the C\textsuperscript{2}U column is highlighted.\par}
\end{table*}

\begin{table*}[t]
\centering
\scriptsize
\setlength{\tabcolsep}{2pt}
\renewcommand{\arraystretch}{0.95}
\caption{\textbf{Control/Neighbor $\Delta$ for every dataset $\times$ encoder cell} (locked test stage; per-cell mean over seeds 1--3). Lower values indicate less non-target degradation; the utility-clean gate requires CN $\le 10^{-9}$. Marking as in Table~\ref{tab:full_h5_drop}.}
\label{tab:full_cn_delta}
\begin{tabular}{@{}ll>{\columncolor{clsProfile!22}}c>{\columncolor{clsProfile!10}}c>{\columncolor{clsLabeled!10}}c>{\columncolor{clsLabeled!10}}c>{\columncolor{clsLabeled!10}}c>{\columncolor{clsLabeled!10}}c>{\columncolor{clsProfile!10}}c>{\columncolor{clsAdapt!10}}c>{\columncolor{clsProfile!10}}c>{\columncolor{clsProfile!10}}c>{\columncolor{clsCand!10}}c>{\columncolor{clsMember!10}}c@{}}
\toprule
\textbf{Dataset} & \textbf{Encoder} & \cellcolor{white}\textbf{C\textsuperscript{2}U} & \cellcolor{white}\textbf{Raw} & \cellcolor{white}\textbf{INLP} & \cellcolor{white}\textbf{LEACE} & \cellcolor{white}\textbf{\makecell{RAG-Emb-\\Ctrl}} & \cellcolor{white}\textbf{CURE} & \cellcolor{white}\textbf{PISCES} & \cellcolor{white}\textbf{\makecell{DARE-\\GRL}} & \cellcolor{white}\textbf{Random} & \cellcolor{white}\textbf{\makecell{Delta-\\Edit}} & \cellcolor{white}\textbf{\makecell{Filter-and-\\Refill$^{\ddagger}$}} & \cellcolor{white}\textbf{\makecell{Rank-\\Router$^{\S}$}} \\
\midrule
\multirow{4}{*}{\textbf{Amazon}}
& minilm & \best{-0.0011} & \secon{0.0000} & 0.0000 & -0.0011 & -0.0011 & 0.0000 & \secon{0.0000} & -0.0011 & \secon{0.0000} & \secon{0.0000} & 0.0011 & -0.0011 \\
& bge & \best{-0.0167} & 0.0000 & -0.0144 & -0.0178 & -0.0178 & 0.0000 & \secon{-0.0067} & 0.0056 & -0.0056 & 0.0000 & -0.0167 & -0.0167 \\
& e5 & \best{-0.0033} & 0.0000 & -0.0011 & -0.0022 & -0.0033 & 0.0000 & \secon{-0.0011} & -0.0033 & 0.0000 & 0.0000 & -0.0033 & -0.0033 \\
& qwen3 & \best{-0.0044} & 0.0000 & -0.0044 & -0.0033 & -0.0056 & 0.0000 & \secon{-0.0022} & -0.0056 & -0.0011 & 0.0000 & -0.0033 & -0.0044 \\
\midrule
\multirow{4}{*}{\textbf{ArXiv}}
& minilm & \best{-0.0122} & 0.0000 & 0.0178 & 0.0233 & 0.0978 & 0.0000 & \secon{-0.0044} & 0.6544 & 0.0044 & 0.0000 & -0.0167 & -0.0089 \\
& bge & \best{-0.0078} & 0.0000 & 0.0422 & 0.0256 & 0.0100 & 0.0000 & \secon{-0.0011} & 0.7244 & 0.0156 & 0.0000 & -0.0189 & -0.0111 \\
& e5 & \best{-0.0144} & 0.0000 & 0.1189 & 0.0178 & 0.0478 & 0.0000 & \secon{-0.0022} & 0.6211 & 0.1467 & 0.0000 & -0.0200 & -0.0067 \\
& qwen3 & \best{-0.0156} & 0.0000 & 0.0211 & 0.0211 & 0.0667 & 0.0000 & \secon{-0.0022} & 0.7433 & 0.0300 & 0.0011 & -0.0233 & -0.0167 \\
\midrule
\multirow{4}{*}{\textbf{IMDB}}
& minilm & \best{0.0000} & \best{0.0000} & 0.0000 & 0.0000 & 0.0000 & 0.0000 & \best{0.0000} & 0.3333 & \secon{0.0578} & \best{0.0000} & 0.0000 & 0.0000 \\
& bge & \best{0.0000} & \best{0.0000} & 0.0000 & 0.0000 & 0.0000 & 0.0000 & \best{0.0000} & 0.3333 & \secon{0.0022} & \best{0.0000} & 0.0000 & 0.0000 \\
& e5 & \best{0.0000} & \best{0.0000} & 0.0000 & 0.0000 & 0.0000 & 0.0000 & \best{0.0000} & 0.0000 & \secon{0.0211} & \best{0.0000} & 0.0000 & 0.0000 \\
& qwen3 & \best{0.0000} & \best{0.0000} & 0.0000 & 0.0000 & 0.0000 & 0.0000 & \best{0.0000} & 0.3333 & \best{0.0000} & \best{0.0000} & 0.0000 & 0.0000 \\
\midrule
\multirow{4}{*}{\textbf{MNLI}}
& minilm & \best{-0.0244} & 0.0000 & -0.0256 & -0.0067 & 0.0389 & 0.0000 & \secon{-0.0078} & 0.2322 & 0.0167 & 0.0011 & -0.0422 & -0.0256 \\
& bge & \best{-0.0100} & 0.0000 & -0.0211 & -0.0033 & 0.0078 & 0.0000 & \secon{-0.0022} & 0.2578 & 0.0256 & 0.0000 & -0.0211 & -0.0078 \\
& e5 & \best{-0.0156} & \secon{0.0000} & 0.0356 & 0.0000 & -0.0100 & 0.0000 & \secon{0.0000} & 0.3278 & 0.0356 & \secon{0.0000} & -0.0178 & -0.0200 \\
& qwen3 & \secon{0.0067} & \best{0.0000} & 0.0200 & 0.0122 & 0.0167 & 0.0000 & \secon{0.0067} & 0.2911 & 0.0289 & \best{0.0000} & 0.0033 & 0.0044 \\
\midrule
\multirow{4}{*}{\textbf{TOFU}}
& minilm & \best{-0.0067} & 0.0000 & 0.0278 & -0.0000 & -0.0133 & 0.0000 & 0.0011 & 0.1833 & \secon{-0.0033} & 0.0000 & -0.0078 & -0.0067 \\
& bge & \best{-0.0011} & \secon{0.0000} & 0.0467 & 0.0000 & -0.0067 & 0.0000 & 0.0011 & 0.2400 & 0.0122 & \secon{0.0000} & -0.0067 & -0.0011 \\
& e5 & \best{-0.0011} & \secon{0.0000} & 0.1122 & -0.0067 & -0.0156 & 0.0000 & \best{-0.0011} & 0.2567 & 0.0233 & \secon{0.0000} & -0.0089 & -0.0011 \\
& qwen3 & \best{-0.0089} & \secon{0.0000} & 0.0722 & 0.0033 & -0.0122 & 0.0000 & 0.0022 & 0.3033 & 0.0056 & \secon{0.0000} & -0.0244 & -0.0078 \\
\bottomrule
\end{tabular}
\par\vspace{1mm}
{\raggedright\footnotesize \textit{Notes:} Each cell is the mean over the three locked seeds (1--3). Encoders: minilm = all-MiniLM-L6-v2, bge = bge-base-en-v1.5, e5 = e5-base-v2, qwen3 = Qwen3-Embedding-0.6B (pinned revisions). Raw is the identity (no-edit) reference. Marking scope per row is the exact profile-only access class (C\textsuperscript{2}U, Raw, PISCES, Random, Delta-Edit); INLP and LEACE (400 labeled calibration examples per unit), CURE and RAG-Emb-Ctrl (400 labeled calibration examples per unit; posthoc adapted-baseline extensions under the same locked protocol), DARE-GRL (training-based), $^{\ddagger}$ (candidate/document access), and $^{\S}$ (target-membership label-assisted) are shown unmarked. Utility-clean requires CN $\le 10^{-9}$ and TE $\le 10^{-9}$, and no cross-access-class dominance is claimed. ArXiv uses a relaxed global control; IMDB carries the duplicate-template diagnostic label. Column tints encode the locked access class (same colors as \cref{fig:pareto_panels}); the C\textsuperscript{2}U column is highlighted.\par}
\end{table*}

\begin{table*}[t]
\centering
\scriptsize
\setlength{\tabcolsep}{2pt}
\renewcommand{\arraystretch}{0.95}
\caption{\textbf{Target Exposure $\Delta$ for every dataset $\times$ encoder cell} (locked test stage; per-cell mean over seeds 1--3). Lower values indicate less false-positive retrieval of target documents by control queries; the utility-clean gate requires TE $\le 10^{-9}$. Marking as in Table~\ref{tab:full_h5_drop}.}
\label{tab:full_te_delta}
\begin{tabular}{@{}ll>{\columncolor{clsProfile!22}}c>{\columncolor{clsProfile!10}}c>{\columncolor{clsLabeled!10}}c>{\columncolor{clsLabeled!10}}c>{\columncolor{clsLabeled!10}}c>{\columncolor{clsLabeled!10}}c>{\columncolor{clsProfile!10}}c>{\columncolor{clsAdapt!10}}c>{\columncolor{clsProfile!10}}c>{\columncolor{clsProfile!10}}c>{\columncolor{clsCand!10}}c>{\columncolor{clsMember!10}}c@{}}
\toprule
\textbf{Dataset} & \textbf{Encoder} & \cellcolor{white}\textbf{C\textsuperscript{2}U} & \cellcolor{white}\textbf{Raw} & \cellcolor{white}\textbf{INLP} & \cellcolor{white}\textbf{LEACE} & \cellcolor{white}\textbf{\makecell{RAG-Emb-\\Ctrl}} & \cellcolor{white}\textbf{CURE} & \cellcolor{white}\textbf{PISCES} & \cellcolor{white}\textbf{\makecell{DARE-\\GRL}} & \cellcolor{white}\textbf{Random} & \cellcolor{white}\textbf{\makecell{Delta-\\Edit}} & \cellcolor{white}\textbf{\makecell{Filter-and-\\Refill$^{\ddagger}$}} & \cellcolor{white}\textbf{\makecell{Rank-\\Router$^{\S}$}} \\
\midrule
\multirow{4}{*}{\textbf{Amazon}}
& minilm & \best{-0.0700} & 0.0000 & 0.0811 & 0.0178 & -0.3533 & 0.0000 & 0.0067 & 0.3678 & 0.0189 & \secon{-0.0011} & -0.0100 & -0.1011 \\
& bge & \best{-0.1356} & \secon{0.0000} & 0.2922 & -0.0611 & -0.2644 & 0.0000 & 0.0111 & 0.3956 & 0.0200 & \secon{0.0000} & -0.0933 & -0.1411 \\
& e5 & \best{-0.1800} & \secon{0.0000} & 0.2733 & -0.0267 & -0.2556 & 0.0000 & 0.0644 & 0.3778 & 0.0600 & \secon{0.0000} & -0.0256 & -0.1833 \\
& qwen3 & \best{-0.1122} & 0.0000 & 0.1078 & 0.0367 & -0.4678 & 0.0000 & 0.0256 & 0.4000 & 0.0211 & \secon{-0.0022} & -0.0367 & -0.1278 \\
\midrule
\multirow{4}{*}{\textbf{ArXiv}}
& minilm & \best{-0.0389} & 0.0000 & 0.0567 & 0.1922 & -0.1756 & 0.0000 & \secon{-0.0033} & 0.5067 & 0.0211 & 0.0000 & -0.1467 & -0.0389 \\
& bge & \best{-0.0878} & \secon{0.0000} & 0.0500 & 0.1789 & -0.1789 & 0.0000 & 0.0078 & 0.4700 & 0.0333 & \secon{0.0000} & -0.1478 & -0.0878 \\
& e5 & \best{-0.1089} & \secon{0.0000} & 0.1900 & 0.1422 & -0.2578 & 0.0000 & 0.0400 & 0.7056 & 0.1489 & \secon{0.0000} & -0.2156 & -0.1089 \\
& qwen3 & \best{-0.0700} & \secon{0.0000} & 0.0589 & 0.1344 & -0.2289 & 0.0000 & 0.0089 & 0.6622 & 0.0489 & \secon{0.0000} & -0.2122 & -0.0700 \\
\midrule
\multirow{4}{*}{\textbf{IMDB}}
& minilm & \best{0.0000} & \best{0.0000} & 0.6667 & 0.3333 & 0.0000 & 0.0000 & \best{0.0000} & 1.0000 & \secon{0.2822} & \best{0.0000} & 0.0000 & 0.0000 \\
& bge & \best{0.0000} & \best{0.0000} & 0.0000 & 0.0000 & 0.3333 & 0.0000 & \best{0.0000} & 0.6667 & \secon{0.0944} & \best{0.0000} & 0.0000 & 0.0000 \\
& e5 & \best{0.0000} & \best{0.0000} & 0.0000 & 0.0000 & 0.3333 & 0.0000 & \best{0.0000} & 0.3333 & \secon{0.3167} & \best{0.0000} & 0.0000 & 0.0000 \\
& qwen3 & \best{0.0000} & \best{0.0000} & 0.0000 & 0.0000 & 0.3333 & 0.0000 & \best{0.0000} & 0.3333 & \secon{0.2311} & \best{0.0000} & 0.0000 & 0.0000 \\
\midrule
\multirow{4}{*}{\textbf{MNLI}}
& minilm & \best{-0.0556} & 0.0000 & 0.0378 & 0.0433 & -0.3789 & 0.0000 & \secon{-0.0078} & -0.1211 & -0.0044 & 0.0000 & -0.1411 & -0.0578 \\
& bge & \best{-0.0611} & 0.0000 & 0.0111 & 0.0433 & -0.3833 & 0.0000 & \secon{-0.0011} & 0.0722 & 0.0167 & 0.0000 & -0.1544 & -0.0644 \\
& e5 & \best{-0.0911} & \secon{0.0000} & 0.1400 & 0.1333 & -0.3011 & 0.0000 & 0.0211 & 0.2500 & 0.0756 & \secon{0.0000} & -0.1456 & -0.0978 \\
& qwen3 & \best{-0.0333} & \secon{0.0000} & 0.0444 & 0.0978 & -0.3067 & 0.0000 & 0.0044 & 0.3789 & 0.0589 & 0.0011 & -0.0322 & -0.0344 \\
\midrule
\multirow{4}{*}{\textbf{TOFU}}
& minilm & \best{-0.0200} & \secon{0.0000} & 0.0633 & 0.0211 & -0.1844 & 0.0000 & 0.0022 & -0.2378 & 0.0100 & \secon{0.0000} & -0.0611 & -0.0278 \\
& bge & \best{-0.0244} & \secon{0.0000} & 0.1111 & 0.0167 & -0.2433 & 0.0000 & 0.0144 & 0.0811 & 0.0100 & 0.0011 & -0.0433 & -0.0267 \\
& e5 & \best{-0.0167} & \secon{0.0000} & 0.1122 & -0.0233 & -0.2133 & 0.0000 & 0.0056 & 0.3356 & 0.0222 & \secon{0.0000} & -0.0811 & -0.0244 \\
& qwen3 & \best{-0.0244} & 0.0000 & 0.0733 & 0.0089 & -0.3100 & 0.0000 & -0.0011 & 0.2900 & \secon{-0.0056} & 0.0000 & -0.1067 & -0.0278 \\
\bottomrule
\end{tabular}
\par\vspace{1mm}
{\raggedright\footnotesize \textit{Notes:} Each cell is the mean over the three locked seeds (1--3). Encoders: minilm = all-MiniLM-L6-v2, bge = bge-base-en-v1.5, e5 = e5-base-v2, qwen3 = Qwen3-Embedding-0.6B (pinned revisions). Raw is the identity (no-edit) reference. Marking scope per row is the exact profile-only access class (C\textsuperscript{2}U, Raw, PISCES, Random, Delta-Edit); INLP and LEACE (400 labeled calibration examples per unit), CURE and RAG-Emb-Ctrl (400 labeled calibration examples per unit; posthoc adapted-baseline extensions under the same locked protocol), DARE-GRL (training-based), $^{\ddagger}$ (candidate/document access), and $^{\S}$ (target-membership label-assisted) are shown unmarked. Utility-clean requires CN $\le 10^{-9}$ and TE $\le 10^{-9}$, and no cross-access-class dominance is claimed. ArXiv uses a relaxed global control; IMDB carries the duplicate-template diagnostic label. Column tints encode the locked access class (same colors as \cref{fig:pareto_panels}); the C\textsuperscript{2}U column is highlighted.\par}
\end{table*}

\cref{tab:rag_comparison_full} retains the complete dataset--method rendering.

\cref{tab:full_h5_drop,tab:full_tc_drop,tab:full_cn_delta,tab:full_te_delta} report the complete dataset $\times$ encoder decomposition behind the pooled cells of \cref{tab:rag_comparison_concise}, one grid per metric; every cell is the mean over the three locked seeds, and the per-seed rows are retained in the locked artifact. Utility-clean status and the marking rule follow the pre-registered Pareto criterion of \cref{sec:exp_concept_erasure}: CN $\le 10^{-9}$ and TE $\le 10^{-9}$, H5 Drop compared only inside the exact profile-only class, and no cross-class dominance; bold and underline mark the best and second-best values inside the profile-only access class of each row. ArXiv carries the \emph{relaxed global control} label and IMDB the \emph{duplicate-template diagnostic} label throughout.

\section{Visual Diagnostics}
\label{app:visual_diagnostics}

The five panels in this section are descriptive renderings of locked outcomes under the unified figure standard of this paper: they recompute or replot rows that the locked chains already contain, and they add no new evidence. Each caption states its populations, denominators, and provenance. The paragraphs below connect each panel to the main-text claim it supports and state what the panel does not show.

\begin{figure*}[t]
\centering
\includegraphics[width=\textwidth]{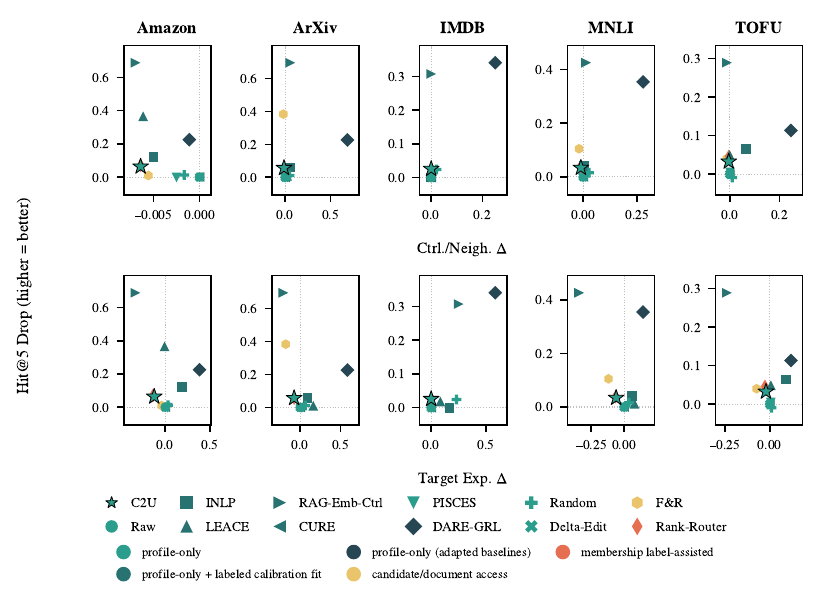}
\caption{\textbf{Full suppression--utility Pareto view.} All twelve displayed methods on every dataset, pooled over four encoders and three seeds, with access-class colors. CURE coincides with Raw because its utility gate locks it to identity. Stronger suppression and its costs remain visible.}
\label{fig:pareto_panels}
\end{figure*}

\paragraph{Pareto view.}
\cref{fig:pareto_panels} plots every displayed method's pooled suppression against its pooled utility cost, one point per method and dataset, pooled over four encoders and three seeds, with colors marking the locked access classes. CURE coincides with Raw because its utility gate locks it to a zero-strength identity, and DARE-GRL's larger suppression sits together with its large positive CN and TE costs.

\begin{figure*}[t]
  \centering
  \phantomcaption
  \begin{subfigure}{\textwidth}
    \centering
    \includegraphics[width=\textwidth]{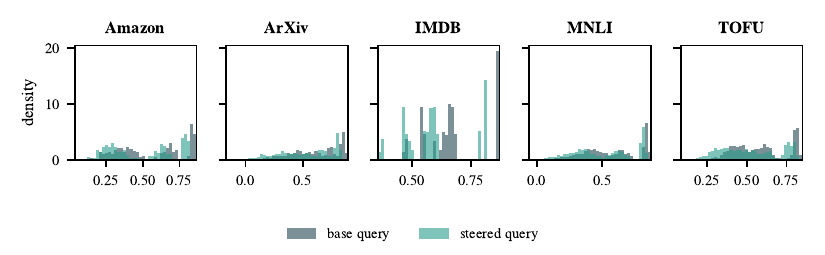}
    \subcaption{\textbf{Target-cosine distributions before and after steering} (locked test rows). The steered distribution shifts left of base in every dataset; \cref{fig:projection} adds the control distribution.}
    \label{fig:target_cosine_shift}
  \end{subfigure}

  \begin{subfigure}{\textwidth}
    \centering
    \includegraphics[width=\textwidth]{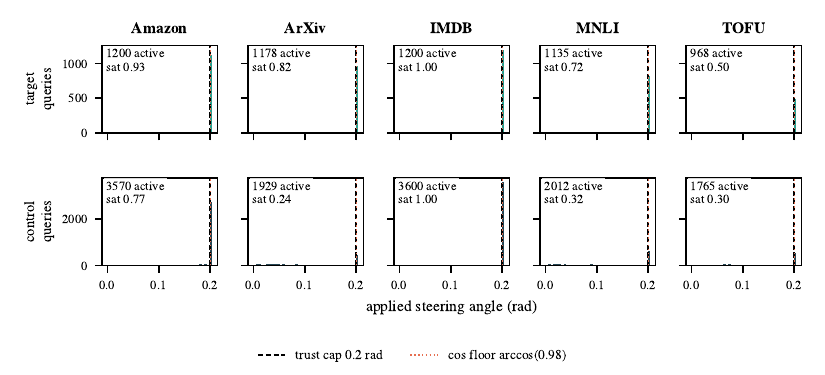}
    \subcaption{\textbf{Applied-angle distribution under the delta-exit trust region} (locked test rows). Active edits concentrate at the default $0.2$-rad cap, slightly tighter than $\arccos(0.98)$; among edited queries, $80.7\%$ of the $5{,}681$ target edits and $61.9\%$ of the $12{,}876$ control edits sit at the cap, whereas \cref{fig:editing_analysis} uses all queries of each group as denominator. Gate-off rows are zero. Rows denote query role and columns dataset. These are descriptive locked-outcome distributions.}
    \label{fig:applied_angle_hist}
  \end{subfigure}
\end{figure*}

\paragraph{Target-cosine shift.}
\cref{fig:target_cosine_shift} shows per-query cosines to the target-document centroid on locked test rows, base versus steered, one panel per dataset. The steered target distribution shifts left of base in every dataset: queries move away from the target centroid, which is the geometric reading of the Hit@5 drops in \cref{tab:rag_comparison_concise}. The figure shows target rows only; \cref{fig:projection} adds the control distribution, and the main-text \cref{fig:editing_analysis} is the pooled three-population companion.

\paragraph{Applied-angle readout.}
\cref{fig:applied_angle_hist} is the trust region's readout: the angle the deployed edit actually applies to each locked test query. Gate-off rows are exactly zero, since a closed gate is an exact identity. Active edits concentrate at the $0.2$-rad cap, slightly tighter than $\arccos(0.98)$; among edited queries, $80.7\%$ of the $5{,}681$ target edits and $61.9\%$ of the $12{,}876$ control edits sit at the cap, whereas the main-text \cref{fig:editing_analysis} uses all queries of each group as denominator ($76.5\%$ and $44.3\%$). Saturation at the cap is what the operating-point sweep of \cref{app:method_details} stress-tests.

\begin{figure*}[t]
  \centering
  \phantomcaption
  \begin{subfigure}{\textwidth}
    \centering
    \includegraphics[width=0.68\textwidth]{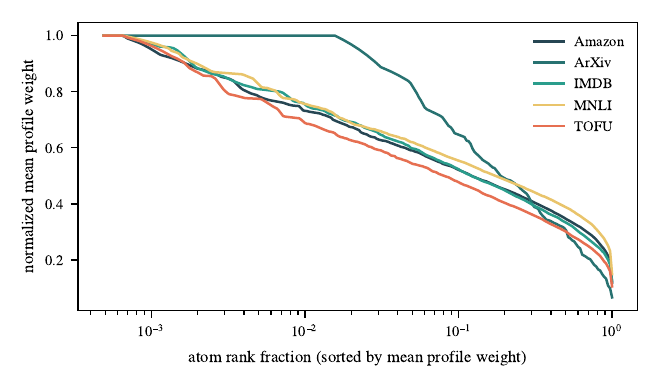}
    \subcaption{\textbf{Concept-profile sharpness across datasets} (locked validation assets). Each curve is the median over a dataset's 12 locked validation units of the sorted, max-normalized mean profile-weight vector, against atom-rank fraction. ArXiv uses $64$ atoms for every encoder; the other datasets use $768$/$1{,}536$/$2{,}048$ for MiniLM/bge+e5/Qwen3. Sharper drop-offs indicate concentrated profile weights, not sparse deployed edits. One per-concept row is appended at deployment (\cref{ssec:deployment_config}).}
    \label{fig:sharpness}
  \end{subfigure}

  \begin{subfigure}{\textwidth}
    \centering
    \includegraphics[width=0.85\textwidth]{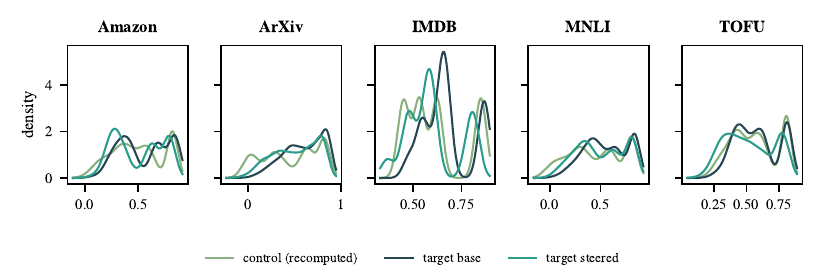}
    \subcaption{\textbf{Projection-score separability before and after steering} (locked test outcomes). Per-dataset kernel densities of each query's cosine to the target centroid of the unchanged base index: target at base, target after steering, and control (recomputed from the locked per-unit tensors; the same recomputation reproduces the locked target rows to $2.8\times10^{-8}$). A steered target distribution shifting toward the control distribution indicates suppression of concept expression.}
    \label{fig:projection}
  \end{subfigure}
\end{figure*}

\paragraph{Profile sharpness.}
\cref{fig:sharpness} measures how concentrated each dataset's concept profiles are in atom coordinates: each curve is the median over a dataset's $12$ locked validation units of the sorted, max-normalized mean profile-weight vector, plotted against atom-rank fraction. ArXiv uses $64$ atoms for every encoder; the other datasets use $768$/$1{,}536$/$2{,}048$ for MiniLM/bge+e5/Qwen3. Sharper drop-offs mean that profile mass concentrates on a few atoms, which is why one per-concept residual row appended at deployment can represent the profile direction in bank coordinates (\cref{ssec:deployment_config}). Concentrated profiles do not by themselves imply sparse deployed edits.

\paragraph{Projection separability.}
\cref{fig:projection} plots kernel densities of each query's cosine to the target centroid of the unchanged base index for three populations: target at base, target after steering, and control, recomputed from the locked per-unit tensors (the same recomputation reproduces the locked target rows to $2.8\times10^{-8}$). A steered target distribution that moves toward the control distribution reads as suppressed concept expression. Cosine to a single centroid is one direction only, so the panel corroborates the hit-based outcomes rather than certifying suppression by itself.

\section{Multi-Concept Composition Ladder}
\label{app:multi_concept}

\cref{tab:multi_concept_ladder} retains the four-arm history. Only the expanded ArXiv $K=4$ arm is a pass; its $K=2$ design and the earlier arms are validation-only or failures and cannot be presented as passes. Each concept has a separate raw-space G1/G0 gate and coefficient vector. If $s_c$ are gates, start with $t_c=0.3s_c$ and set $B_q=\min(0.3\max_c s_c,0.2)$. For positive total request allocate $\tilde t_c=t_c\min(1,B_q/\sum_\ell t_\ell)$; if all gates vanish all allocations are zero. Per-concept angular feasibility and sine-ratio scaling precede additive delta composition and the shared trust exit. The comparator uses a normalized average of unit O2 directions with a single max gate. The endpoint takes minima of per-concept H5 within each method and unit, then differences them and averages units, as in \eqref{eq:multi_endpoint}; it does not minimize paired conceptwise differences or use a min gate.

\begin{table*}[t]
\centering
\scriptsize
\setlength{\tabcolsep}{2pt}
\caption{\textbf{Multi-concept composition ladder.} Primary endpoint: paired per-unit difference of min-concept Hit@5 Drop, GR-C2U-ACR-compose minus the internal ablation GR-C2U-Fixed-m3-multi (positive favors compose). v1 (original pools) passes validation only; v2 (new pools) fails selection and carries a registered split-overlap construction defect; v3 (cursor-disjoint) is scope-limited to ArXiv and fails selection on both designs; v1--v3 touch no test rows and their test stages stay hard-closed. v4 (cursor-disjoint tier-T3 capacity expansion on fresh rows) fails selection at $K{=}2$ but passes selection, confirmation, and the hard-gated test stage at $K{=}4$. F\&R-multi is a document-membership oracle upper reference, non-gating.}
\label{tab:multi_concept_ladder}
\begin{tabular}{lllcccc}
\toprule
\textbf{Arm} & \textbf{Pools} & \textbf{Design} & \textbf{\makecell{Selection $\Delta$ \\{}[CI95]}} & \textbf{\makecell{Confirmation $\Delta$ \\{}[CI95]}} & \textbf{\makecell{Test $\Delta$ \\{}[CI95]}} & \textbf{Verdict} \\
\midrule
v1 & original & \makecell[l]{MNLI $K{=}4$; \\ ArXiv $K{=}4$} & \makecell{+0.0300 \\{}[+0.0167, +0.0450]} & \makecell{+0.0250 \\{}[+0.0050, +0.0467]} & --- & \textbf{PASS}\textsuperscript{v} \\
v2 & new\textsuperscript{d} & \makecell[l]{Amazon $K{=}4$; \\ ArXiv $K{=}8$; \\ MNLI $K{=}3$} & \makecell{+0.0244 \\{}[+0.0122, +0.0389]} & -0.0044\textsuperscript{c} & --- & FAIL\textsuperscript{d} \\
v3 & cursor-disjoint & ArXiv $K{=2}$\textsuperscript{e} & +0.0100\textsuperscript{g} & ---\textsuperscript{f} & --- & FAIL\textsuperscript{d} \\
v3 & cursor-disjoint & ArXiv $K{=4}$\textsuperscript{e} & +0.0067\textsuperscript{g} & ---\textsuperscript{f} & --- & FAIL\textsuperscript{d} \\
v4 & cursor-disjoint & ArXiv $K{=}2$\textsuperscript{e} & \makecell{-0.0008 \\{}[-0.0075, +0.0058]} & ---\textsuperscript{f} & --- & FAIL\textsuperscript{d} \\
v4 & cursor-disjoint & ArXiv $K{=}4$\textsuperscript{e} & \makecell{+0.0425 \\{}[+0.0279, +0.0575]} & \makecell{+0.0517 \\{}[+0.0375, +0.0663]} & \makecell{+0.0412 \\{}[+0.0325, +0.0506]} & \textbf{PASS}\textsuperscript{t} \\
\bottomrule
\end{tabular}
\par\vspace{1mm}
{\raggedright\footnotesize \textit{Notes:} \textsuperscript{v} validation-only acceptance; not a test claim. \textsuperscript{d} diagnostic-only; the test runner refuses unless the method lock is PASS. \textsuperscript{c} v2 confirmation CI95 straddles zero per dataset: Amazon [-0.0233, +0.0033]; ArXiv [-0.0100, +0.0100]; MNLI [-0.0100, +0.0000]; v2 selection per-concept compose Hit@5 Drop on Amazon L1 is negative for Books (-0.0133) and Electronics (-0.0267), failing the per-concept gate. Granularity: the v2 Amazon cells use L1; all ArXiv cells (v2, v3, v4) use L2. \textsuperscript{e} Concepts (pre-registered free-row ranking; v4 re-ranked after v3 consumption): v3 $K{=}2$: \{Artificial Intelligence, Software Engineering\}; $K{=}4$: \{Artificial Intelligence, Genomics, Neurons and Cognition, Software Engineering\}; v4 $K{=}2$: \{Software Engineering, Artificial Intelligence\}; $K{=}4$: \{Software Engineering, Artificial Intelligence, Neurons and Cognition, Number Theory\}. \textsuperscript{f} confirmation never ran: no design passed selection. \textsuperscript{g} v3 gates failed (CI95 lower bound $\leq 0$; $K{=}2$: seeds 2/3, encoders 3/4; $K{=}4$: seeds 2/3, encoders 1/4); v4 $K{=}2$ fails the CI-lower-bound, all-seeds, and encoder gates (seeds 1/3, encoders 2/4 positive). \textsuperscript{t} test-stage acceptance on freshly allocated, never-before-used test rows (the first use of these test rows in any experiment, evaluated exactly once with no test-time tuning); the refusal path was exercised and evidenced before the PASS method lock existed. v4 tier T3 (pre-registered before any v4 outcome): 200/200 selection/confirmation and 400 test target queries per concept, 1{,}200 control rows per split; v4 $K{=}4$ test gates: seeds 3/3, encoders 4/4 positive. Filter-and-Refill-multi (document-membership oracle, non-gating upper reference) min-concept Hit@5 Drop on the same units: v2 confirmation Amazon +0.6967, ArXiv +0.9333, MNLI +0.8867; v3 selection $K{=}2$ +0.9433, $K{=}4$ +0.8967; v4 selection $K{=}2$ +0.7683, $K{=}4$ +0.7662; v4 test $K{=}4$ +0.7181. v3/v4 scope: MNLI/IMDB/TOFU excluded by capacity (best-case free control rows MNLI 395, IMDB 565, TOFU 428 vs 1,200 required), Amazon by mechanism (v2-measured L1-granularity retrieval-level ineffectiveness; L2 offers no non-target control).\par}
\end{table*}

\paragraph{Original pools (v1): validation-only pass.}
Pooled over MNLI and ArXiv $K=4$, the paired acceptance claim passes selection ($+0.0300$, CI95 $[+0.0167,+0.0450]$) and independent confirmation ($+0.0250$, CI95 $[+0.0050,+0.0467]$). All eight dataset--concept compose H5 cells are positive. Other datasets are capacity-blocked, not zeros; this pass authorizes no test by itself.

\paragraph{New pools (v2): failed acceptance with a registered defect.}
Amazon L1 $K=4$, ArXiv L2 $K=8$, and MNLI $K=3$ have a positive selection paired endpoint ($+0.0244$), but selection acceptance fails: Amazon Books and Electronics have compose H5 $-0.0133$ and $-0.0267$. Confirmation does not replicate ($-0.0044$). The builder restarts allocation cursors across role kinds, contaminating selection with profile rows and creating corpus/query overlap. Confirmation is disjoint from profiles, and its independent failure supports the negative verdict; test rows are disjoint from other slices and never accessed. The overlap remains an additional blocker, not an excused positive result. IMDB/TOFU control capacity is insufficient. The diagnostic ArXiv $K=2/4/8$ sweep is also negative and carries the defect label.

\paragraph{Cursor-disjoint rebuild (v3): failed selection.}
The repaired builder uses advancing cross-role cursors and true overlap checks. Pre-outcome capacity/mechanism rules restrict it to ArXiv: MNLI/IMDB/TOFU lack controls, while Amazon is excluded after ineffective L1 retrieval suppression and lacks suitable L2 controls. Both pre-registered $K=2$ and $K=4$ designs fail selection. Confirmation never runs and test refusal is exercised; no test rows are touched.

\paragraph{Capacity expansion (v4): ArXiv $K=4$ passes.}
Earlier locked validation differences inform a budget design, not a causal diagnosis of quantization noise. The T1/T2/T3 ladder is frozen before v4 outcomes; all tiers fail the predeclared $0.8$ power criterion. The adjudicated fallback selects the largest feasible T3: $200$ selection, $200$ confirmation, and $400$ test target queries per concept, with $1{,}200$ controls per split. Profiles remain $20$ documents per concept; the index uses $25$ target documents per concept and $300$ controls. Cursor-disjoint allocation requires $845$ rows per concept and $4{,}800$ control rows overall. Method configuration and the atom bank are unchanged, but evaluation pools and budgets differ.

Concepts are frozen by free-row ranking after prior consumption: Software Engineering, Artificial Intelligence, Neurons and Cognition, Number Theory; $K=2$ is the prefix. $K=2$ fails selection ($-0.0008$, CI95 $[-0.0075,+0.0058]$, only one of three seeds positive). $K=4$ passes selection ($+0.0425$, CI95 $[+0.0279,+0.0575]$) and independent confirmation ($+0.0517$, CI95 $[+0.0375,+0.0663]$). Its hard-gated test endpoint is $+0.04125$, CI95 $[+0.0325,+0.050625]$ from $4{,}000$ bootstraps of twelve paired unit differences; all four concept cells, three seeds, and four encoders are positive. The retained table rounds these to $+0.0412$ and $[+0.0325,+0.0506]$. These are previously unused multi-concept test rows, first evaluated without tuning after the passing lock; refusal was exercised before that lock existed. This untouched-test scope is separate from the posthoc single-concept/attack cohorts. The document-membership oracle Filter-and-Refill-multi is a non-gating upper reference ($+0.7181$), not an access-matched competitor.

The ladder supports a scoped improvement over the averaged-direction comparator on ArXiv $K=4$. It does not prove that larger budgets remove interference, that noise caused earlier failures, or that atom representation alone causes the gain.

\begin{figure*}[t]
\centering
\includegraphics[width=\textwidth]{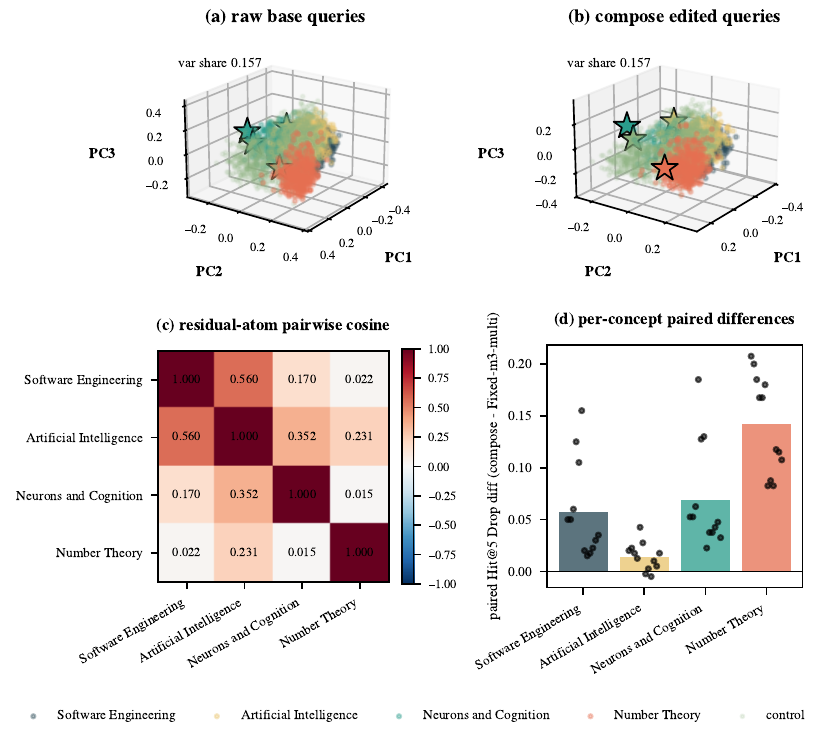}
\caption{\textbf{Descriptive composition view of locked ArXiv $K=4$ outcomes.} (a,b) Raw and edited queries in one deterministic PCA basis on the pre-registered bge/seed-1 illustration (variance share $0.157$); centroid shifts are $0.022$ for controls and $0.058$--$0.118$ for targets. (c) Residual-atom cosines, not the Gram matrix of full O2 edit directions: SE--NT $0.022$, NC--NT $0.015$, SE--AI $0.560$. (d) Paired per-concept H5 differences; AI is smallest ($+0.0142$). Selection/confirmation/test endpoints are $+0.0425/0.0517/0.0412$ (test CI95 $[+0.0325,+0.0506]$); oracle refill is $+0.7181$. $K=2$ failed selection and never entered test. This preserves the first-use, no-test-tuning scope of the locked $K=4$ test.}
\label{fig:concept_composition_viz}
\end{figure*}

\paragraph{Composition view.}
\cref{fig:concept_composition_viz} is a descriptive view of the locked ArXiv $K=4$ outcomes, not new evidence. Panels (a,b) place raw and compose-edited queries in one deterministic PCA basis on the pre-registered bge/seed-1 illustration (variance share $0.157$): control centroids shift by $0.022$, target centroids by $0.058$--$0.118$. Panel (c) reports residual-atom cosines---SE--NT $0.022$, NC--NT $0.015$, SE--AI $0.560$---not the Gram matrix of full O2 edit directions. Panel (d) shows paired per-concept H5 differences, with Artificial Intelligence smallest ($+0.0142$): every concept contributes positively, which is what the min-concept endpoint of \eqref{eq:multi_endpoint} rewards. $K=2$ failed selection and never entered test.

\section{Query-Transform Evaluation and Full Results}
\label{app:attack_protocol}

\begin{table}[t]
\centering
\small
\definecolor{arrProfile}{HTML}{2A9D8C}
\definecolor{arrLabeled}{HTML}{287271}
\definecolor{arrAdapt}{HTML}{264653}
\definecolor{arrCandidate}{HTML}{E9C46B}
\definecolor{arrMember}{HTML}{E66F51}
\definecolor{arrHero}{HTML}{E9C46B}
\caption{\textbf{Cross-condition behavior over the fifteen dataset--condition cells.} Suppressing: density $\Delta<0$; divergent: suppressing with CER $\Delta>0$; worst values are signed maxima over the fifteen cells (pp). Counts use unrounded values; the per-condition cells are in Table~\ref{tab:attack_full}.}
\label{tab:attack_counts}
\setlength{\tabcolsep}{2pt}
\begin{tabular}{@{}lcccc@{}}
\toprule
Method & \shortstack{Suppressing\\/15} & \shortstack{Divergent\\/15} & \shortstack{Worst\\CER $\Delta$} & \shortstack{Worst\\utility loss} \\
\midrule
\rowcolor{arrHero!25} C\textsuperscript{2}U & 15/15 & 0/15 & 0.0 & +2.6 \\
\rowcolor{arrProfile!10} PISCES & 14/15 & 11/15 & +4.9 & +0.1 \\
\rowcolor{arrProfile!10} Random & 12/15 & 11/15 & +22.3 & +5.7 \\
\rowcolor{arrLabeled!10} INLP & 15/15 & 13/15 & +25.0 & +5.4 \\
\rowcolor{arrLabeled!10} LEACE & 15/15 & 13/15 & +12.5 & +3.4 \\
\rowcolor{arrLabeled!10} RAG-Emb-Ctrl & 15/15 & 3/15 & +16.7 & +14.0 \\
\rowcolor{arrLabeled!10} CURE & 0/15 & 0/15 & 0.0 & 0.0 \\
\rowcolor{arrAdapt!10} DARE-GRL & 11/15 & 11/15 & +50.9 & +66.3 \\
\bottomrule
\end{tabular}
\end{table}

\cref{tab:attack_full} retains every dataset--condition value and the Mean columns; \cref{tab:attack_counts} summarizes the same cells as counts and signed worst values. The original arm covers $5$ datasets $\times4$ encoders $\times3$ seeds $\times3$ conditions $\times7$ methods, including Raw: $504{,}000$ per-query rows. A separate additive arm supplies the two recent adaptations; the original count is not the combined-arm count. Both are posthoc extensions of frozen single-concept units, with no test tuning and no untouched-test claim. The main table omits Raw's construction-zero row, not its role as the paired reference. Candidate/document-access Filter-and-Refill, membership-assisted Rank-Router, and Delta-Edit are predeclared attack exclusions.

\begin{table*}[t]
\scriptsize
\centering
\definecolor{clsProfile}{HTML}{2A9D8C}
\definecolor{clsLabeled}{HTML}{287271}
\definecolor{clsAdapt}{HTML}{264653}
\definecolor{clsHero}{HTML}{E9C46B}
\caption{\textbf{Retrieval under standard and transformed queries.} Per-query paired percentage-point changes relative to the Raw reference under the \emph{same} condition; \textit{Target density$\Delta$}, \textit{CER$\Delta$} on control queries (context budget $B{=}512$, at most $50$ documents), and \textit{control own-label Hit@10 loss} are all lower-is-better. Bold/underline are computed only inside the exact profile-only group (C\textsuperscript{2}U, PISCES, Random), excluding the Raw reference; no cross-access-group dominance is claimed. The Mean columns pool the three query conditions per dataset.}
\label{tab:attack_full}
\setlength{\tabcolsep}{2pt} %
\begin{tabular}{@{}lcccccccccccc@{}}
\toprule
 & \multicolumn{3}{c}{\textbf{\shortstack{Standard\\(no attack)}}} & \multicolumn{3}{c}{\textbf{Injection}} & \multicolumn{3}{c}{\textbf{Zou-noise}} & \multicolumn{3}{c}{\textbf{Mean}} \\
\cmidrule(lr){2-4}\cmidrule(lr){5-7}\cmidrule(lr){8-10}\cmidrule(lr){11-13}
\textbf{Method} & \textbf{\shortstack{Target\\density$\Delta$}} & \textbf{\shortstack{Control\\CER$\Delta$}} & \textbf{\shortstack{Hit@10\\loss$\downarrow$}} & \textbf{\shortstack{Target\\density$\Delta$}} & \textbf{\shortstack{Control\\CER$\Delta$}} & \textbf{\shortstack{Hit@10\\loss$\downarrow$}} & \textbf{\shortstack{Target\\density$\Delta$}} & \textbf{\shortstack{Control\\CER$\Delta$}} & \textbf{\shortstack{Hit@10\\loss$\downarrow$}} & \textbf{\shortstack{Target\\density$\Delta$}} & \textbf{\shortstack{Control\\CER$\Delta$}} & \textbf{\shortstack{Hit@10\\loss$\downarrow$}} \\
\midrule
\multicolumn{13}{@{}l}{\textbf{Amazon}} \\
\midrule
\rowcolor{clsHero!30} C\textsuperscript{2}U & \best{-18.0} & \best{-12.2} & \best{-0.1} & \best{-14.3} & \best{-16.7} & \best{0.0} & \best{-16.7} & \best{-12.4} & \best{0.0} & \best{-16.4} & \best{-13.7} & \best{0.0} \\
\rowcolor{clsProfile!10} PISCES & -3.5 & +3.4 & \best{-0.1} & \secon{+0.2} & \secon{+4.9} & \best{0.0} & \secon{-0.5} & \secon{+4.6} & \best{0.0} & \secon{-1.2} & \secon{+4.3} & \best{0.0} \\
\rowcolor{clsProfile!10} Random & \secon{-5.7} & \secon{+2.2} & \secon{0.0} & +2.5 & +11.2 & \best{0.0} & +2.1 & +5.7 & \best{0.0} & -0.4 & +6.4 & \secon{0.0} \\
\rowcolor{clsLabeled!10} INLP & -31.2 & +20.1 & -0.1 & -5.4 & +16.9 & 0.0 & -10.7 & +15.8 & 0.0 & -15.8 & +17.6 & 0.0 \\
\rowcolor{clsLabeled!10} LEACE & -45.2 & 0.0 & -0.1 & -6.2 & +3.2 & 0.0 & -14.3 & +4.1 & 0.0 & -21.9 & +2.4 & 0.0 \\
\rowcolor{clsLabeled!10} RAG-Emb-Ctrl & -59.7 & -34.6 & -0.1 & -19.3 & -28.6 & 0.0 & -25.1 & -30.8 & 0.0 & -34.7 & -31.4 & 0.0 \\
\rowcolor{clsLabeled!10} CURE & 0.0 & 0.0 & 0.0 & 0.0 & 0.0 & 0.0 & 0.0 & 0.0 & 0.0 & 0.0 & 0.0 & 0.0 \\
\rowcolor{clsAdapt!10} DARE-GRL & -38.4 & +36.3 & -0.1 & +9.5 & +50.7 & 0.0 & +1.3 & +50.9 & 0.0 & -9.2 & +46.0 & 0.0 \\
\midrule
\multicolumn{13}{@{}l}{\textbf{ArXiv}} \\
\midrule
\rowcolor{clsHero!30} C\textsuperscript{2}U & \best{-14.6} & \best{-6.1} & \best{-0.8} & \best{-20.1} & \best{-21.9} & \best{-5.2} & \best{-18.4} & \best{-7.1} & \best{-0.8} & \best{-17.7} & \best{-11.7} & \best{-2.2} \\
\rowcolor{clsProfile!10} PISCES & -2.1 & \secon{+1.1} & \secon{-0.2} & -2.0 & \secon{-0.4} & \secon{-0.2} & -2.0 & \secon{+1.0} & \secon{-0.2} & -2.0 & \secon{+0.6} & \secon{-0.2} \\
\rowcolor{clsProfile!10} Random & \secon{-4.8} & +4.8 & +3.6 & \secon{-6.0} & +3.9 & +5.7 & \secon{-6.0} & +5.1 & +4.7 & \secon{-5.6} & +4.6 & +4.7 \\
\rowcolor{clsLabeled!10} INLP & -33.0 & +3.7 & +4.6 & -35.0 & -3.3 & +2.3 & -34.4 & +5.6 & +5.4 & -34.1 & +2.0 & +4.1 \\
\rowcolor{clsLabeled!10} LEACE & -14.8 & +10.6 & +1.6 & -11.8 & +12.5 & +3.4 & -15.4 & +12.3 & +1.7 & -14.0 & +11.8 & +2.2 \\
\rowcolor{clsLabeled!10} RAG-Emb-Ctrl & -68.2 & -14.0 & +5.6 & -76.3 & -34.9 & +3.3 & -72.7 & -16.8 & +6.9 & -72.4 & -21.9 & +5.3 \\
\rowcolor{clsLabeled!10} CURE & 0.0 & 0.0 & 0.0 & 0.0 & 0.0 & 0.0 & 0.0 & 0.0 & 0.0 & 0.0 & 0.0 & 0.0 \\
\rowcolor{clsAdapt!10} DARE-GRL & -51.4 & +41.5 & +66.3 & -50.5 & +19.6 & +60.0 & -50.4 & +40.2 & +64.6 & -50.7 & +33.8 & +63.6 \\
\midrule
\multicolumn{13}{@{}l}{\textbf{IMDB}} \\
\midrule
\rowcolor{clsHero!30} C\textsuperscript{2}U & \best{-10.9} & \best{0.0} & \best{0.0} & \best{-14.5} & \best{0.0} & \best{0.0} & \best{-7.9} & \best{-2.3} & \best{-0.1} & \best{-11.1} & \best{-0.8} & \best{0.0} \\
\rowcolor{clsProfile!10} PISCES & 0.0 & \best{0.0} & \best{0.0} & -0.4 & \best{0.0} & \best{0.0} & -1.4 & \secon{+0.7} & \secon{+0.1} & -0.6 & \secon{+0.2} & \secon{0.0} \\
\rowcolor{clsProfile!10} Random & \secon{-8.1} & \secon{+8.2} & \secon{+0.1} & \secon{-1.1} & \secon{+22.3} & \secon{+1.7} & \secon{-7.0} & +13.6 & +0.3 & \secon{-5.4} & +14.7 & +0.7 \\
\rowcolor{clsLabeled!10} INLP & -18.6 & 0.0 & 0.0 & -11.8 & +25.0 & 0.0 & -19.6 & +5.3 & -0.2 & -16.7 & +10.1 & -0.1 \\
\rowcolor{clsLabeled!10} LEACE & -4.7 & +8.3 & 0.0 & -7.2 & 0.0 & 0.0 & -5.0 & +8.9 & +0.4 & -5.6 & +5.7 & +0.1 \\
\rowcolor{clsLabeled!10} RAG-Emb-Ctrl & -35.5 & +8.3 & 0.0 & -31.0 & +16.7 & 0.0 & -36.3 & +1.6 & -0.2 & -34.3 & +8.9 & -0.1 \\
\rowcolor{clsLabeled!10} CURE & 0.0 & 0.0 & 0.0 & 0.0 & 0.0 & 0.0 & 0.0 & 0.0 & 0.0 & 0.0 & 0.0 & 0.0 \\
\rowcolor{clsAdapt!10} DARE-GRL & -38.3 & +25.0 & 0.0 & -15.4 & +33.3 & 0.0 & -34.4 & +25.3 & 0.0 & -29.4 & +27.9 & 0.0 \\
\midrule
\multicolumn{13}{@{}l}{\textbf{MNLI}} \\
\midrule
\rowcolor{clsHero!30} C\textsuperscript{2}U & \best{-11.4} & \best{-3.4} & \best{-0.5} & \best{-19.8} & \best{-8.8} & +2.6 & \best{-13.5} & \best{-8.1} & \best{-1.1} & \best{-14.9} & \best{-6.8} & \secon{+0.3} \\
\rowcolor{clsProfile!10} PISCES & -2.0 & \secon{+0.9} & \secon{-0.2} & -2.0 & +0.4 & \best{-0.2} & -2.8 & \secon{+0.1} & \secon{-0.3} & -2.3 & \secon{+0.4} & \best{-0.2} \\
\rowcolor{clsProfile!10} Random & \secon{-4.8} & +2.9 & +1.0 & \secon{-3.5} & \secon{-0.3} & \secon{+1.4} & \secon{-3.2} & +0.9 & +1.3 & \secon{-3.9} & +1.2 & +1.2 \\
\rowcolor{clsLabeled!10} INLP & -16.9 & +9.0 & -0.9 & -14.8 & +3.0 & -1.5 & -21.2 & +2.3 & -0.3 & -17.7 & +4.8 & -0.9 \\
\rowcolor{clsLabeled!10} LEACE & -8.0 & +6.6 & -0.4 & -7.5 & +1.1 & +0.1 & -9.6 & +2.2 & -0.7 & -8.4 & +3.3 & -0.3 \\
\rowcolor{clsLabeled!10} RAG-Emb-Ctrl & -43.1 & -45.1 & +0.6 & -37.8 & -51.8 & +14.0 & -40.7 & -45.5 & +2.6 & -40.6 & -47.5 & +5.7 \\
\rowcolor{clsLabeled!10} CURE & 0.0 & 0.0 & 0.0 & 0.0 & 0.0 & 0.0 & 0.0 & 0.0 & 0.0 & 0.0 & 0.0 & 0.0 \\
\rowcolor{clsAdapt!10} DARE-GRL & -37.5 & +13.1 & +10.0 & -21.9 & +2.3 & +0.3 & -30.0 & +4.3 & +10.4 & -29.8 & +6.6 & +6.9 \\
\midrule
\multicolumn{13}{@{}l}{\textbf{TOFU}} \\
\midrule
\rowcolor{clsHero!30} C\textsuperscript{2}U & \best{-2.9} & \best{-1.6} & \best{-0.1} & \best{-5.7} & \best{-6.1} & \secon{0.0} & \best{-4.4} & \best{-5.8} & \best{-0.3} & \best{-4.4} & \best{-4.5} & \best{-0.1} \\
\rowcolor{clsProfile!10} PISCES & \secon{-0.9} & \secon{+1.1} & \secon{+0.1} & \secon{-0.8} & \secon{+1.3} & \best{-0.1} & \secon{-0.3} & +1.7 & \secon{+0.1} & \secon{-0.6} & \secon{+1.3} & \secon{+0.1} \\
\rowcolor{clsProfile!10} Random & -0.1 & +1.1 & +0.9 & -0.5 & +2.2 & +1.0 & +0.2 & \secon{+1.4} & +0.8 & -0.1 & +1.5 & +0.9 \\
\rowcolor{clsLabeled!10} INLP & -7.0 & +11.4 & +1.8 & -5.7 & +14.4 & +2.6 & -4.2 & +16.3 & +2.1 & -5.6 & +14.0 & +2.2 \\
\rowcolor{clsLabeled!10} LEACE & -5.7 & +1.4 & -0.3 & -5.5 & +3.9 & +0.2 & -4.2 & +4.4 & +0.1 & -5.1 & +3.2 & 0.0 \\
\rowcolor{clsLabeled!10} RAG-Emb-Ctrl & -18.3 & -27.4 & -0.4 & -22.8 & -34.7 & +1.9 & -18.9 & -29.8 & -0.4 & -20.0 & -30.6 & +0.4 \\
\rowcolor{clsLabeled!10} CURE & 0.0 & 0.0 & 0.0 & 0.0 & 0.0 & 0.0 & 0.0 & 0.0 & 0.0 & 0.0 & 0.0 & 0.0 \\
\rowcolor{clsAdapt!10} DARE-GRL & -1.2 & +15.3 & +13.3 & +0.3 & +20.2 & +13.6 & +0.2 & +19.7 & +14.6 & -0.3 & +18.4 & +13.8 \\
\bottomrule
\end{tabular}
\par\vspace{1mm}
{\raggedright\footnotesize \textit{Note:} Filter-and-Refill and Rank-Router are excluded on access-class mismatch (candidate/document and membership-label access versus this arm's query-only contract). Raw (ref) is exactly zero by construction (method minus itself); its all-zero reference rows are omitted from the display. This arm is a posthoc extension of the frozen test units, evaluated exactly once with no test-time tuning; no untouched-test claim is made. CURE's all-zero cells are its locked zero-strength utility-gate identity, verified by execution. Row tints encode the locked access class (same colors as \cref{fig:pareto_panels}); the C\textsuperscript{2}U row is highlighted.\par}
\end{table*}

Transforms are bitwise-equal reimplementations of the original pure functions; the recent-baseline arm replays its standard-condition runs unchanged, and CURE's all-zero cells are its locked zero-strength identity. Aggregates use $4{,}000$ paired bootstrap samples. CER is averaged over control queries with the $512$-token/$50$-document packing rule; the legacy attack ``Hit'' heading is clarified as target density without changing values. Main \cref{tab:attack_baselines_cer_compact} summarizes equal-weight cell means, and \cref{tab:attack_counts} the unrounded suppression/divergence signs and signed worst values; neither is significance or adaptive-attack certification.

\paragraph{Reading the full matrix.}
\cref{tab:attack_full} expands every summary cell into per-condition values. Each dataset block lists the eight evaluated methods in the main table's order with locked access-class tints and the C\textsuperscript{2}U row highlighted; each condition group pairs target density $\Delta$ on target queries with CER $\Delta$ and own-label Hit@10 loss on control queries, and the Mean group reproduces the main table's per-dataset cells. Marks compare only the exact profile-only group (C\textsuperscript{2}U, PISCES, Random).

\paragraph{Condition-level patterns.}
Read per condition, the counts in \cref{tab:attack_counts} are not artifacts of Mean pooling. C\textsuperscript{2}U lowers target density in all $15$ cells and never raises control exposure, with attack utility between $-5.2$ and $+2.6$ pp; its remaining variation is condition mix, not sign. INLP's Amazon pair---density down $15.8$ pp, exposure up $17.6$ pp on Mean---repeats under every condition, the displacement pattern behind the $13/15$ divergent cells it shares with LEACE. RAG-Emb-Ctrl suppresses most strongly on ArXiv ($-72.4$ density, $-21.9$ exposure) yet diverges in the three IMDB cells, worst $+16.7$ pp under injection, and costs utility on ArXiv ($+5.3$ pp) and MNLI ($+5.7$); the duplicate-template limitation applies. CURE's all-zero matrix is its locked zero-strength identity, so its $0/15$ divergence count is trivial. PISCES, Random, and DARE-GRL diverge in $11/15$ cells.

\clearpage
\section{Reproducibility and Source Record}
\label{sec:reproducibility}

Frozen experimental packages retain per-arm protocols, configuration/input/method locks, encoder/environment records, per-query outputs, aggregation rules, and independent checks. Bootstrap units differ: main per-dataset H5 intervals reuse the standard replay's paired-row bootstrap, whereas the multi-concept test resamples twelve unit differences. Recent-baseline provenance replay is explicitly retrospective. Unknown legacy resource measurements remain unavailable, not estimated.

The ARR manuscript uses the official ACL template and a lightweight single-JSON conversion/source record, \texttt{CONVERSION\_PROVENANCE.json}. That record is distinct from frozen experimental audit chains; ARR has no ICLR-style closure-manifest claim. The restructuring introduces no experiments or test-time selection. Full matrices, negative probes, and descriptive views distinguish evidence reuse from new inference. Generated summaries derive from existing locked aggregates rather than newly accessed raw test queries.

\section{Theoretical Analysis of the Deployed Edit Operator}
\label{app:theory}

We distinguish exact real-arithmetic identities from the finite-precision, thresholded implementation in \cref{ssec:deployment_config}. No theorem below proves retrieval benefit, atom necessity, or absence of multi-concept interference. Throughout, $\mathbb S^{d-1}$ is the unit sphere of $\mathbb R^d$, $g$ denotes the negative diagonal-O2 unit direction, and $d$ denotes ambient dimension. Unit queries and documents are assumed where stated. Numerical branches are specified separately in \cref{app:method_details}.

\subsection{Trust-Region Geometry of the Delta Exit}

\begin{definition}[Edit segment and ideal trust region]
\label{def:edit_operator}
For $q\in\mathbb S^{d-1}$ and finite $\delta\ne0$, define
\begin{equation}
 E_\eta(q)=\frac{q+\eta\delta}{\|q+\eta\delta\|_2},\qquad \eta\in[0,1],
\end{equation}
where the denominator is nonzero. For $0<\gamma_{\rm floor}<1$ and $0<\theta_{\rm cap}\le\pi$, define
\begin{equation}
\begin{gathered}
 \gamma=\max(\cos\theta_{\rm cap},\gamma_{\rm floor}),\\
 \mathcal T_\gamma(q)=\{u\in\mathbb S^{d-1}:\langle u,q\rangle\ge\gamma\}.
\end{gathered}
\end{equation}
Then $\theta_{\max}=\arccos\gamma=\min(\theta_{\rm cap},\arccos\gamma_{\rm floor})$. The angular request is already incorporated into $\delta$; the exit never increases its step above one. On a nonsingular segment the ideal exit chooses the largest feasible step in $[0,1]$. Finite bisection need not return that exact real-valued endpoint.
\end{definition}

\begin{lemma}[Monotone exit on a bounded segment]
\label{lem:monotone_exit}
If $\delta$ is not parallel to $q$, set $a=\langle\delta,q\rangle$ and $b=\|\delta\|^2$. Then $\varphi(\eta)=\langle E_\eta(q),q\rangle$ is continuous on $[0,1]$ and strictly decreasing for positive steps. If $\varphi(1)\ge\gamma$, the full step is feasible. Otherwise a unique $\eta_*\in(0,1)$ satisfies $\varphi(\eta_*)=\gamma$; exact-arithmetic bisection brackets it, with feasible lower endpoints converging to $\eta_*$.
\end{lemma}
\begin{proof}
Non-parallelism gives $b-a^2>0$, so $q+\eta\delta$ cannot vanish. Direct differentiation gives
\begin{equation}
\begin{gathered}
 \varphi(\eta)=\frac{1+\eta a}{\sqrt{1+2\eta a+\eta^2b}},\\
 \varphi'(\eta)=\frac{-\eta(b-a^2)}{(1+2\eta a+\eta^2b)^{3/2}}.
\end{gathered}
\end{equation}
The derivative is negative for $\eta>0$, and $\varphi(0)=1>\gamma$. Continuity gives the unique crossing when $\varphi(1)<\gamma$. Maintaining feasible lower and infeasible upper endpoints gives the bisection claim.
\end{proof}

\begin{remark}[Degeneracy and numerical scope]
If $\delta=tq$, the normalized point stays $q$ while $1+\eta t>0$; an antiparallel segment may reach a singular zero or flip sign. Such a segment is not covered by the lemma. Finite FP32 checks and denominator safeguards do not establish exact real-number boundary attainment or monotonicity of rounded cosines. The drift result below applies to a normalized output satisfying its stated cap condition; numerical tolerances must be accounted for when applying it to computed outputs.
\end{remark}

\begin{proposition}[Uniform worst-case drift bound]
\label{prop:drift_bound}
Assume $q,q'\in\mathbb S^{d-1}$ and $q'\in\mathcal T_\gamma(q)$. For every unit vector $u \in \mathbb{S}^{d-1}$, in particular every frozen document embedding,
\begin{equation}
\begin{gathered}
\bigl| \langle q', u \rangle - \langle q, u \rangle \bigr|
\;\le\; \|q' - q\|_2
\\
\;=\; 2 \sin \frac{\angle(q', q)}{2}
\;\le\; 2 \sin \frac{\theta_{\max}}{2}.
\end{gathered}
\end{equation}
\end{proposition}

\begin{proof}
The inequality $|\langle q' - q, u \rangle| \le \|q' - q\|_2 \|u\|_2 = \|q' - q\|_2$ is Cauchy--Schwarz. For two unit vectors separated by angle $\theta$, the chord length is $\|q' - q\|_2 = \sqrt{2 - 2 \cos \theta} = 2 \sin(\theta / 2)$, and $\theta \le \theta_{\max}$ holds because $q' \in \mathcal{T}_{\gamma}(q)$ by construction; sine is increasing on $[0, \pi/2]$, which contains $\theta_{\max}$ since $\gamma > 0$.
\end{proof}

\begin{corollary}[Margin stability of the top-ranked document]
\label{cor:margin}
Under the hypotheses of \cref{prop:drift_bound}, let the fixed index contain at least two unit document vectors, with $u_1,u_2$ ranked first and second for $q$. Set $m=\langle q,u_1\rangle-\langle q,u_2\rangle$. If $m>4\sin(\theta_{\max}/2)$, then $u_1$ remains uniquely top-ranked after editing.
\end{corollary}

\begin{proof}
By \cref{prop:drift_bound} each score moves by at most $2 \sin(\theta_{\max}/2)$, so the post-edit gap is at least $m - 4 \sin(\theta_{\max}/2) > 0$.
\end{proof}

\begin{remark}[Floor and cap]
If $\theta_{\rm cap}>\arccos\gamma_{\rm floor}$, the floor is tighter; at equality they coincide. Below that angle the nominal cap is tighter. Equal floors across raised caps do not imply equality with a smaller-cap edit, nor identical rankings or H5. The observed operating-point ties and the residual-coefficient counterexample are empirical results in \cref{app:method_details}.
\end{remark}

\subsection{Exact Direction Routing through the Residual Atom}

\begin{definition}[Bank projection and completed coordinates]
\label{def:residual_atom}
For $\mathcal A\in\mathbb R^{N\times d}$ and nonzero $g$, let
\begin{equation}
\begin{gathered}
 a=(\mathcal A^\top)^+g,\quad P=\mathcal A^\top(\mathcal A^\top)^+,\\
 p=Pg,\qquad e=(I-P)g.
\end{gathered}
\end{equation}
Here $M^+$ is the Moore--Penrose pseudoinverse. If $e\ne0$, append $r=e/\|e\|$ and use $b_\beta=[a;\beta\|e\|]$. If $e=0$, then $p=g\ne0$: append $r=p/\|p\|$ and use $b_\beta=[a;0]$. In both cases $\mathcal A'=[\mathcal A;r]$ is defined.
\end{definition}

\begin{proposition}[Projected plus residual recovers the full direction]
\label{prop:exact_routing}
In exact arithmetic, $e$ is orthogonal to the row span of $\mathcal A$ and
\begin{equation}
 \mathcal A'^\top b_\beta=p+\beta e.
\end{equation}
For $\beta=1$ this is full $g$, including when $g$ is in span. The projected coefficients $a$ are the minimum-norm preimage of $p$. Appending $r$ raises rank by one exactly when $e\ne0$; in the in-span branch it adds no rank.
\end{proposition}
\begin{proof}
$P$ is the orthogonal projector onto $\operatorname{range}(\mathcal A^\top)$, hence $p\perp e$ and $\mathcal A^\top a=p$. If $e\ne0$, the last row contributes $\beta\|e\|r=\beta e$, so the sum is $p+\beta e$; at $\beta=1$, $p+e=g$. If $e=0$, the last coefficient is zero and $p=g$. Pseudoinverse identities give the minimum-norm preimage property. A nonzero orthogonal row raises rank by one, whereas normalized $p$ is already in the row span.
\end{proof}

\begin{remark}[What residual completion does not say]
A pure last-coordinate edit produces only a multiple of $e$ when $e\ne0$. That is not the deployed full-direction route, which also changes projected coefficients $a$. At a positive numerical threshold $\varepsilon_r$, the fallback for $0<\|e\|<\varepsilon_r$ yields $p$, not exactly $g$, even before rounding; its exact-arithmetic truncation error is $\|e\|$. FP32 bank/weight subtraction adds further error. The theorem describes ideal routing, not bitwise equality or a retrieval advantage over dense $g$.
\end{remark}

\paragraph{Angular scale.}
For unit direction $h$, $q^\top h=\cos\phi$ with $0<\phi<\pi$, and $0\le\theta<\phi$, solving for a turn of $\theta$ along $q+kh$ gives $k=\sin\theta/\sin(\phi-\theta)$. Write $h=\cos\phi\,q+\sin\phi\,v_\perp$ and resolve the angle to obtain $k\sin(\phi-\theta)=\sin\theta$. The runtime uses this scale with unnormalized routed coefficients. At $\beta=1$ their ideal dense direction has unit norm; for general $\beta$, $\|p+\beta e\|$ need not be one, so the same coefficient scale need not realize precisely the requested angle. The trust exit acts afterward in either case.

\subsection{Coherence Bounds for Multi-Concept Composition}

For simultaneous suppression of $K\ge2$ target concepts (\cref{ssec:multi_concept_main}), let $\delta_1,\dots,\delta_K\in\mathbb S^{d-1}$ be the unit full per-concept edit directions and write $\Delta=\sum_{k=1}^K\eta_k\delta_k$ for nonnegative coefficients. Here $\eta_k$ denotes each pre-exit dense magnitude after angular allocation, not the exit step $\eta$ or the allocated angle itself. The following is an exact-arithmetic bound on this sum.

\begin{definition}[Mutual coherence]
\label{def:coherence}
The \emph{Gram matrix} of the concept directions is $G \in \mathbb{R}^{K \times K}$, $G_{k\ell} = \langle \delta_k, \delta_\ell \rangle$, and the \emph{mutual coherence} is $\mu = \max_{k \ne \ell} |G_{k\ell}|$.
\end{definition}

\begin{proposition}[Cross-talk and spectral sandwich]
\label{prop:coherence}
With the notation above:
\begin{enumerate}[label=(\alph*),leftmargin=2em]
\item each concept direction retains alignment $\langle \Delta, \delta_k \rangle \ge \eta_k - \mu \sum_{\ell \ne k} \eta_\ell$;
\item every eigenvalue $\lambda$ of $G$ lies in $\bigl[\, 1 - (K-1)\,\mu,\; 1 + (K-1)\,\mu \,\bigr]$, and therefore
\begin{equation}
\begin{gathered}
\max\bigl(0,\, 1 - (K-1)\,\mu\bigr) \, \|\eta\|_2^2
\\
\;\le\; \|\Delta\|_2^2 = \eta^{\top} G \, \eta
\\
\;\le\; \bigl(1 + (K-1)\,\mu\bigr) \, \|\eta\|_2^2;
\end{gathered}
\end{equation}
\item in the equal-strength case $\eta_k = \eta > 0$, the condition $\mu < \frac{1}{K-1}$ suffices for the composite edit to keep positive alignment with every concept direction.
\end{enumerate}
\end{proposition}

\begin{proof}
(a) $\langle \Delta, \delta_k \rangle = \eta_k + \sum_{\ell \ne k} \eta_\ell \langle \delta_\ell, \delta_k \rangle \ge \eta_k - \mu \sum_{\ell \ne k} \eta_\ell$, using $\langle \delta_\ell, \delta_k \rangle \ge -\mu$ and $\eta_\ell \ge 0$.
(b) $G$ has unit diagonal and off-diagonal entries bounded by $\mu$ in absolute value, so each Gershgorin disc of $G$ is centered at $1$ with radius $\sum_{\ell \ne k} |G_{k\ell}| \le (K-1)\,\mu$; since $G$ is symmetric, its eigenvalues are real and lie in the union of these discs. The Rayleigh-quotient sandwich $\lambda_{\min}(G) \|\eta\|_2^2 \le \eta^{\top} G \eta \le \lambda_{\max}(G) \|\eta\|_2^2$ with $\lambda_{\min}(G) \ge \max(0, 1 - (K-1)\mu)$ (the Gram matrix is positive semidefinite) and $\eta^{\top} G \eta = \bigl\| \sum_k \eta_k \delta_k \bigr\|_2^2$ gives the claim.
(c) Plug $\eta_k = \eta$ into (a): $\langle \Delta, \delta_k \rangle \ge \eta \bigl(1 - (K-1)\,\mu\bigr) > 0$ for every $k$.
\end{proof}

\begin{remark}[Geometry is not retrieval non-interference]
The norm bound depends on both coherence and coefficient magnitudes. Positive alignment of a pre-normalization sum does not guarantee improved per-concept retrieval after normalization and trust clipping. Nor can the residual-atom heatmap in \cref{fig:concept_composition_viz} instantiate $G$: at $\beta=1$ the full directions contain projected and residual components. With a common exact projector, $\langle g_k,g_\ell\rangle=\langle p_k,p_\ell\rangle+\langle e_k,e_\ell\rangle$, not simply $\langle r_k,r_\ell\rangle$. Low residual coherence therefore does not establish full-direction incoherence, causal isolation of dictionary benefits, or absence of interference.
\end{remark}

\subsection{Weight Concentration and Gate-Off Identity}

The weight generator assigns $w_{x,i}\propto\exp(\operatorname{sim}(h_x,a_i)/\tau)$ over the $N$ atoms (\cref{ssec:ACR}). The next lemma describes input-weight concentration, not locality of the deployed edit: least-squares coefficients can be dense regardless of the support or peakedness of $w_x$.

\begin{lemma}[Temperature concentration of the atom weights]
\label{lem:softmax}
For $N\ge2$ finite real scores $s_i=\operatorname{sim}(h_x,a_i)$ and temperature $\tau>0$, order $s_{(1)}\ge s_{(2)}\ge\cdots\ge s_{(N)}$. Let $\Delta_i=s_{(1)}-s_{(i)}$ and $\Delta=\min_{i\ge2}\Delta_i$ be the top-gap. Then
\begin{equation}
\begin{gathered}
w_{(1)} \;\ge\; \frac{1}{1 + \sum_{i \ge 2} e^{-\Delta_i / \tau}}
\\
\;\ge\; \frac{1}{1 + (N-1)\, e^{-\Delta / \tau}},
\\
\sum_{i \ge 2} w_{(i)} \;\le\; (N-1)\, e^{-\Delta / \tau}.
\end{gathered}
\end{equation}

In particular, at fixed positive top-gap the weights converge to the one-hot argmax at an exponential rate as $\tau \to 0^{+}$.
\end{lemma}

\begin{proof}
Divide numerator and denominator of $w_{(1)}$ by $e^{s_{(1)}/\tau}$ to get $w_{(1)} = \bigl(1 + \sum_{i \ge 2} e^{-\Delta_i/\tau}\bigr)^{-1}$, then bound each $\Delta_i \ge \Delta$ and sum the $N-1$ identical terms; the tail mass is $\sum_{i \ge 2} w_{(i)} = 1 - w_{(1)} \le \frac{(N-1)e^{-\Delta/\tau}}{1 + (N-1)e^{-\Delta/\tau}} \le (N-1)\, e^{-\Delta / \tau}$.
\end{proof}

\begin{proposition}[Gate-off exactness of the implementation branch]
\label{prop:gate_off}
Under the validated finite-input contract, if the deployed single-concept gate is zero, or its resulting dense delta is zero, the runtime returns the original query row without modification. For multi-concept composition, all per-concept gates zero likewise imply an unchanged row.
\end{proposition}

\begin{proof}
A zero gate gives zero angle and coefficient scale; identical padded and steered weights produce zero delta. The exit's active mask requires both a positive gate and a delta norm above its numerical inactivity threshold. For inactive rows it explicitly copies the original source query, bypassing normalization. With every multi-concept gate zero, all allocations and scales vanish and the same copy branch applies. This is bit-exact copying of the input representation, not a consequence of ideal normalization algebra. Historical $\alpha_x$ is absent from this argument and from the deployed operator.
\end{proof}

\begin{remark}
The ideal cap bound and explicit inactive branch explain bounded motion and exact inactivity under their respective hypotheses. They do not establish dataset-level suppression or a dictionary advantage; empirical comparisons remain necessary.
\end{remark}

\end{document}